\documentclass{article}
\usepackage[preprint]{neurips_2025}

\usepackage[utf8]{inputenc} % allow utf-8 input
\usepackage[T1]{fontenc}    % use 8-bit T1 fonts
\usepackage{hyperref}       % hyperlinks
\usepackage{url}            % simple URL typesetting
\usepackage{booktabs}       % professional-quality tables
\usepackage{amsfonts}       % blackboard math symbols
\usepackage{nicefrac}       % compact symbols for 1/2, etc.
\usepackage{microtype}      % microtypography
\usepackage[dvipsnames]{xcolor}         % colors
\usepackage[acronym]{glossaries}
\usepackage{graphicx} %package to manage images
\usepackage{subcaption}
\graphicspath{ {images/} }
\usepackage{natbib}
\usepackage{booktabs}  % For professional-looking tables
\usepackage{multirow}
\usepackage{caption}
\usepackage[none]{hyphenat}
\usepackage{float}
\usepackage{nicefrac}
\usepackage{microtype}
\usepackage{amsthm, amsmath, amssymb}
\usepackage{placeins} % in preamble

\theoremstyle{definition}
\newtheorem{definition}{Definition}
\newtheorem{assumption}{Assumption}

\theoremstyle{plain}
\newtheorem{theorem}{Theorem}

\title{From Grunts to Lexicons: \\Emergent Language from Cooperative Foraging}

\author{%
    Maytus Piriyajitakonkij$^{1,2,3}$ \qquad
    Rujikorn Charakorn$^{5}$ \qquad
    Weicheng Tao$^{2}$ \\
    \textbf{Wei Pan}$^{3}$ \qquad
    \textbf{Mingfei Sun}$^{3}$ \qquad
    \textbf{Cheston Tan}$^{1,4}$ \qquad
    \textbf{Mengmi Zhang}$^{1,2,4}$ \\
    \ \\
    $^{1}$Institute for Infocomm Research (I$^{\text{2}}$R), A$^{*}$STAR, Singapore \\
    $^{2}$College of Computing and Data Science, Nanyang Technological University, Singapore \\
    $^{3}$Department of Computer Science, The University of Manchester, United Kingdom \\
    $^{4}$Centre for Frontier AI Research (CFAR), A$^{*}$STAR, Singapore \\
    $^{5}$Sakana AI, Japan \\
    Address correspondence to \texttt{mengmi.zhang@ntu.edu.sg} \\
}

\begin{document}

\maketitle

\begin{abstract}

%\revise{M2Dx.2 The introduction should be rewritten to align with established anthropological evidence.} 
Language is a powerful communicative and cognitive tool. It enables humans to express thoughts, share intentions, and reason about complex phenomena. Despite our fluency in using and understanding language, the question of how it arises and evolves over time remains unsolved. A leading hypothesis in linguistics and anthropology posits that language evolved to meet the ecological and social demands of early human cooperation. Language did not arise in isolation, but through shared survival goals.
Inspired by this view, we investigate the emergence of language in multi-agent Foraging Games. These environments are designed to reflect the cognitive and ecological constraints believed to have influenced the evolution of communication.
Agents operate in a shared grid world with only partial knowledge about other agents and the environment, and must coordinate to complete games like picking up high-value targets or executing temporally ordered actions. Using end-to-end deep reinforcement learning, agents learn both actions and communication strategies from scratch.
We find that agents develop communication protocols with hallmark features of natural language: arbitrariness, interchangeability, displacement, cultural transmission, and compositionality. We quantify each property and analyze how different factors, such as population size, social dynamics, and temporal dependencies, shape specific aspects of the emergent language.
Our framework serves as a platform for studying how language can evolve from partial observability, temporal reasoning, and cooperative goals in embodied multi-agent settings. We will release all data, code, and models publicly.

\end{abstract}

\vspace{-2mm}
\section{Introduction}
\vspace{-2mm}

%\revise{M2Dx.2 The introduction should be rewritten to align with established anthropological evidence.} %\revise{M2Dx.3: Add might} 
The evolution of human language remains a central open question in science. While humans fluently use and understand language, its origins are still debated. A prominent hypothesis suggests that language emerged through cooperative interaction under partial observability \cite{sterelny2012language, tomasello2010origins, nowak2000evolution, christiansen2003language, nowak2001towards}. Rather than an abstract code, language is viewed as a tool shaped by social use and shared goals \cite{wittgenstein2009philosophical, wagner2003progress}. Although direct evidence from early stages is lacking, multi-agent simulations provide a means to investigate how language might arise from the need to coordinate with others \cite{cangelosi2012simulating, kirby2014iterated, lazaridou2017multi, lazaridou2018emergence}.

%How human language evolved is a fundamental question in science. While humans fluently use and deeply understand language, the question of its origin remains open. A leading hypothesis states that human language evolved not in isolation, but through cooperative interaction under partial observability \cite{sterelny2012language, tomasello2010origins, nowak2000evolution, christiansen2003language, nowak2001towards}. Language is not merely an abstract code but a tool shaped by social use and shared goals \cite{wittgenstein2009philosophical, wagner2003progress}. Although direct evidence from these early stages is unavailable, multi-agent simulations offer a way to study 
%how language might emerge from the need to coordinate with others \cite{cangelosi2012simulating, kirby2014iterated, lazaridou2017multi, lazaridou2018emergence}. 

Prior work on Emergent Communication (EC) has primarily focused on referential games (RG), where a speaker conveys task-relevant information to a listener over a limited communication channel \cite{lazaridou2017multi, lewis2008convention, kharitonov2019egg, gualdoni2024bridging}. These settings have advanced our understanding of symbol grounding and compositionality, but often impose simplifying assumptions: agents are limited to unidirectional communication \cite{galke2022emergent}, operate in fixed speaker-listener roles, and perform disembodied tasks with passive input processing and no active interaction with the environment. Even recent extensions to population-based learning \cite{dubova2020effects, dubova2020reinforcement, kim2021emergent, chaabouni2022emergent, michelrevisiting} typically decouple communication from physical behavior. Such setups diverge from the ecologically grounded, socially interdependent, and embodied conditions under which human language likely evolved \cite{tomasello2010origins, dessalles2007we}.

To bridge this gap, we introduce Foraging Games (FG), a multi-agent framework designed to reflect more realistic ecological and cognitive constraints that may have shaped early human language \cite{dessalles2007we, tomasello2010origins}. FG supports bidirectional communication. Agents are trained to jointly learn both physical action and communication, as early humans likely did during cooperative foraging. The environment enforces embodiment: agents must explore, observe, and act within a dynamic and partially observable world that they can influence. Moreover, each agent interacts with a population of diverse partners, facilitating the study of generalization to new partners, dialect formation, and cultural transmission \cite{kim2021emergent, rita2022role, michelrevisiting}.

Specifically, we formulate FG as a partially observable grid world in which agents must complete multi-step games as shown in \autoref{fig:task}. Success depends not only on taking effective actions, but also on communicating what they see and know. It includes two games: (1) collecting the most valuable item with a partner, which encourages communication about items' scores and locations, and (2) picking up items in a specific order with a partner, which requires communication about when items were seen. Each agent has deep neural network modules for perception, memory, and policies to translate internal representations into actions and messages. Agents communicate through discrete messages drawn from a finite set of learnable vocabulary exchanged at each time step. They are trained independently using Proximal Policy Optimization (PPO) \cite{schulman2017proximal}, with no shared parameters, or gradients, reflecting the personalized and decentralized nature among agents. 

\begin{figure}[t]
    \centering
    \begin{subfigure}{0.28\textwidth}
        \centering
        \includegraphics[trim=0 20 0 0, clip, width=\textwidth]{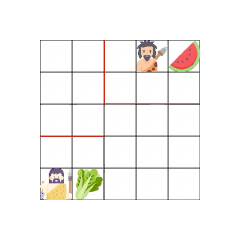}
        \caption{Foraging Games}
        \label{fig:intro_a}
    \end{subfigure}
    \begin{subfigure}{0.26\textwidth}
        \centering
        \includegraphics[trim=0 20 0 0, clip, width=\textwidth]{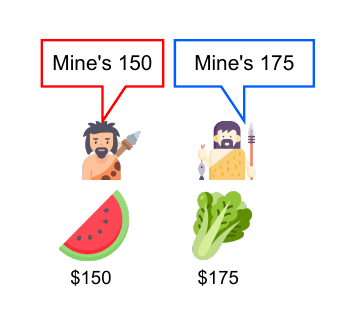}
        \caption{\emph{ScoreG}}
        \label{fig:intro_b}
    \end{subfigure}
    \begin{subfigure}{0.26\textwidth}
        \centering
        \includegraphics[trim=0 20 0 0, clip, width=\textwidth]{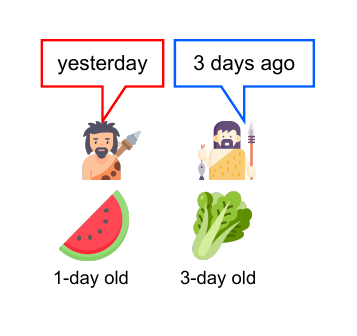}
        \caption{\emph{TemporalG}}
        \label{fig:intro_c}
    \end{subfigure}
    \vspace{-2mm}
    \caption{
    \textbf{An Overview of Foraging Games.} (a) Two agents operate in a $5\times5$ partially observable grid world. Each agent can observe a $3\times3$ grid centered on itself. The two agents are required to pick up the goal items simultaneously. (b) \emph{ScoreG} game: Two items are assigned random scores; each agent observes only one score, and both must pick up the higher-scoring item to succeed in an episode. This game is designed to encourage agents to communicate about items' scores. (c) \emph{TemporalG} game: Two items spawn at random times, each observed by only one agent. Agents must collect them in chronological order, encouraging communication about time.
    }

    \label{fig:task}
    \vspace{-7mm}
\end{figure}

%\revise{M2Dx.1 and vika.4 The research question and hypotheses need to be stated more clearly} 
With this framework in place, we aim to provide answers to the following questions: \textbf{(Q1)} What properties does emergent communication have when it arises in embodied, interactive, and collaborative agents?
\textbf{(Q2)} Do agents truly understand the language they use, and what drives the emergence of a shared language when the same agent acts as both speaker and listener?
\textbf{(Q3)} How do population size and social dynamics affect the emergent language?
\textbf{(Q4)} Can agents refer to spatial and temporal events through their language?
\textbf{(Q5)} Can agents communicate implicitly through non-verbal behaviors, such as body movements?
%\MM{I added all five questions}

Empirically, we find that agents achieve a success rate (SR) above $95\%$ across all games. Although two trained agents perform well when paired with each other, each fails when paired with a copy of itself at test time, suggesting that an agent understands only its partner's language, not the one it produces. We propose two solutions to this problem. The first is to train a population with more than two agents. We hypothesize that developing a shared language becomes the optimal strategy when agents must communicate with multiple partners. The second is to incorporate self-interaction during training \cite{dubova2020effects, charakorn2023generating, charakorn2024diversity}, motivated by the observation that humans can speak to themselves. These solutions encourage agents to comprehend their own messages, promoting convergence on a shared common language and reflecting the property of interchangeability. Since human language develops in populations, we further explore language properties in groups of agents connected via different social structures: fully-connected, ring-structured, and small-world-structured networks. We decode task-relevant information, such as items' positions, scores, and spawn times, from the agents’ messages using logistic regression, verifying that the messages are meaningful rather than random. Above-chance decoding accuracy indicates that a language has emerged to communicate item properties. Specifically, agents develop time adverbials \cite{lipinski2023s} that refer to when past events occurred, along with messages indicating the location of those events. Finally, we show that agents can develop implicit communication \cite{wang2025scout, grupen2022multi, dreyer2025comparing}, i.e., gaining information by observing a partner's behavior, when they are unable to send messages.

Our novel contributions are as follows: 

\noindent \textbf{1.} We introduce Foraging Games, a framework for studying emergent communication that provides ecological and cognitive constraints resembling those faced by early humans, including embodiment, behavior coordination, partial observability, bidirectional communication, and temporal reasoning.

\noindent \textbf{2.} We propose a hybrid cross-and-self-play training regime that enhances cultural transmission in weakly-connected social networks: agents that are closer in the training population develop more similar languages than those farther apart.

\noindent \textbf{3.} We design a series of tasks in FG that elicit both temporal and spatial displacement in emergent language, which could not arise without embodiment.

\vspace{-2mm}
\section{Related Work}
\vspace{-2mm}
EC offers insight into the evolution of human language \cite{lazaridou2017multi, lazaridou2018emergence, kirby2014iterated, chaabouni2022emergent, cangelosi1998emergence, cangelosi2001evolution, boldtreview, peters2025emergent, lazaridou2020emergent} and the development of more effective representations \cite{mu2021emergent, carmelictd, zionsemantics, denamganai2024erelela, yao2022linking}. Early studies employed evolutionary computation \cite{cangelosi1998emergence, cangelosi2001evolution} or Bayesian modeling \cite{nowak2000evolution, kirby2007innateness} to simulate language emergence and its dynamics. The advent of deep reinforcement learning (DRL) \cite{mnih2013playing, mnih2015human} introduced a more powerful framework for studying EC in realistic, complex settings \cite{foerster2016learning}. Contemporary work often applies DRL to Lewis-style referential games (RG) \cite{lewis2008convention}, where a speaker conveys task-relevant information to a listener through a limited communication channel, typically for target identification or input reconstruction \cite{zionsemantics, gualdoni2024bridging, lipinski2024speaking, kharitonov2019egg}. Because language is inherently social, research has shifted toward population-level studies \cite{dubova2020effects, dubova2020reinforcement, kim2021emergent, chaabouni2022emergent, rita2022role, michelrevisiting}, examining how social dynamics and population training influence consistency \cite{dubova2020effects, dubova2020reinforcement} and dialect formation \cite{kim2021emergent}. However, many resulting frameworks resemble autoencoders \cite{kingma2013auto, higgins2017beta, uedalewis}, treating language as a compressed latent code rather than a flexible medium for real-world coordination.

Recent studies have examined more realistic conditions, such as bidirectional exchange, where agents both send and receive messages \cite{evtimova2018emergent, dubova2020effects, dubova2020reinforcement, taillandier2023neural, graesser2019emergent, kottur2017natural, nikolaus2023emergent}, and embodiment \cite{jain2019two, patel2021interpretation, mordatch2018emergence, kajic2020learning}, where agents learn to communicate through interaction with their environment and partners. Only two studies \cite{mordatch2018emergence, jain2019two} incorporate both. Most embodied EC work \cite{jain2019two, patel2021interpretation} prioritizes task performance over analyzing EC itself and omits biologically plausible decentralized training. Another line of research uses EC as a control prompt, outperforming natural language in embodied tasks \cite{mu2023ec2}, though it is still learned through disembodied RG. The closest work to ours is \cite{mordatch2018emergence}, which investigates EC in a physical setting but with all agents sharing a single network, lacking coordinated action toward a shared goal. It also overlooks population training, social dynamics, and key properties such as compositionality, interchangeability, and displacement, while relying on a differentiable environment with direct gradient flow. By contrast, our approach is fully decentralized and reward-driven, with communication and coordination emerging solely from reward, reflecting constraints faced by early humans. We further demonstrate that simple FG yields core linguistic phenomena—most notably displacement, grounded in embodiment.

\begin{figure}[t]
    \centering
    \begin{subfigure}{0.28\textwidth}
        \centering
        \includegraphics[trim=17 17 17 17, clip, width=\textwidth]{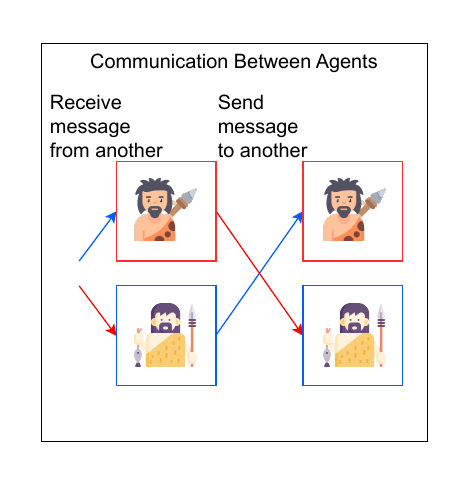}
        \caption{Message Exchange}
        \label{fig:exchange_message}
    \end{subfigure}
    \begin{subfigure}{0.53\textwidth}
        \centering
        \includegraphics[trim=10 10 10 10, clip, width=\textwidth]{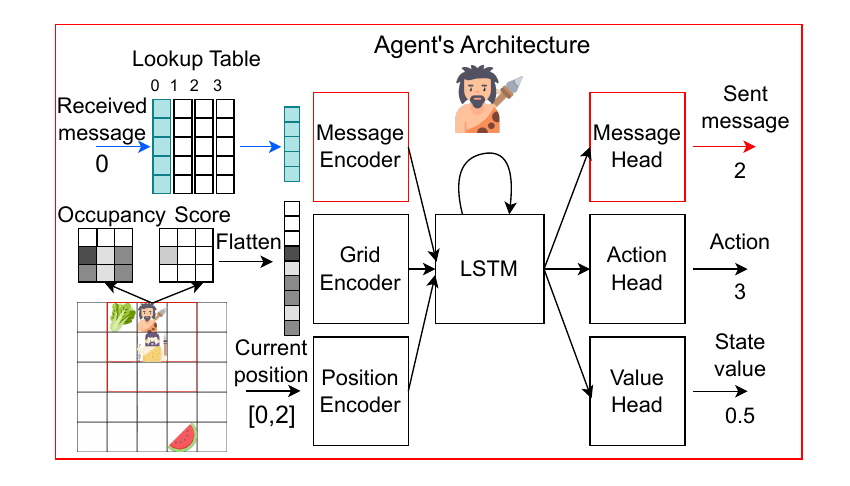}
        \caption{Architecture Overview}
        \label{fig:agent_architecture}
    \end{subfigure}
    \vspace{-2mm}
    \caption{\textbf{A graphical overview of our method.} (a) Agents exchange messages at every time step. (b) The neural architecture of a single agent. On the input side, the received integer message is mapped to a real-valued vector via a learnable lookup table and passed through the message encoder. The grid observation and agent position are processed by the grid encoder and position encoder, respectively. The outputs of all three encoders are concatenated and passed to an LSTM module, which maintains temporal memory. On the output side, the message head, action head, and value head produce the next message token, environment action, and estimated state value.}
    \label{fig:agent_comm}
    \vspace{-5mm}
\end{figure}
\vspace{-3mm}
\section{Foraging Games (FG)}
\label{sec:fg}
\vspace{-3mm}
Foraging-style games are used as benchmarks for multi-agent reinforcement learning (MARL) algorithms \cite{albrecht2013gamefg, papoudakis2021agentfg, yang2022transformerfg, jafferjee2022tamingfg, jaques2019social, ikram2021hexajungle}. However, existing environments often assume full observability and are not designed to support the study of emergent communication. We develop our FG, in which decentralized agents share a common goal and must navigate to and simultaneously pick up the goal items to earn equal rewards.
Our FG is designed to encourage agents to communicate bidirectionally about their observations. Each agent is provided with only partial knowledge of the environment (e.g., observing only one item). Agents must reach a consensus based on their perspectives by communicating the message from a fixed-size dictionary with learnable embeddings and choosing appropriate actions, where the action space is \{move left, move right, move up, move down, and pick up\}. FG has two games: \textit{ScoreG} and \textit{TemporalG}. The former is used to investigate four properties of emergent language: \textbf{interchangeability} (agents can understand the messages they produce), \textbf{arbitrariness} (symbols acquire meaning through social agreement rather than inherent structure), \textbf{compositionality} (messages are built from reusable parts reflecting task semantics), and \textbf{cultural transmission} (language patterns are passed across agents), under varying population sizes and training regimes. The latter focuses on \textbf{displacement}, aiming to determine whether agents’ messages encode temporal and spatial information about seen items in the past.
\vspace{-3mm}
\paragraph{Game 1: \emph{Pickup High Score (ScoreG)}} is set in a $2D$ grid world with the size of $5 \times 5$, containing $2$ items, $2$ agents, and a wall, as shown in \autoref{fig:task}. The wall refers to the area outside the $5 \times 5$ grid. Each grid cell can be occupied by either an agent or an item, but not both simultaneously. Agents cannot move to the wall, the occupied grid cell, or the same grid cell. Each agent receives an observation $\mathbf{o}^{(t)} = (\mathbf{x}^{(t)}, \mathbf{p}^{(t)})$, where $\mathbf{x}^{(t)} \in \mathbb{R}^{3 \times 3 \times 2}$ ($3$ is the receptive field size and 2 refers to two channels). The first channel of the receptive field $\mathbf{x}^{(t)}$ represents an occupancy map: a value of zero indicates that a grid cell contains neither an item nor a wall. 
The second channel shows the item's score. The receptive field of each agent has its center at the agent and observes the agent's surroundings, delineated as a red box in \autoref{fig:task}. To simulate incomplete knowledge, each agent observes the score of one item during an episode. The agent’s absolute position in the environment is given by $\mathbf{p}^{(t)} \in \mathbb{R}^2$. The goal item is defined as the item with the highest score. An episode is deemed successful if both agents navigate to and simultaneously pick up the goal item. In this case, they receive a shared positive reward of $+1$; otherwise, they receive a shared negative reward of $-1$. Additionally, similar to \cite{wijmansdd}, successful episodes grant both agents a bonus reward $r_b$ based on their efficiency: $r_b = \frac{T_{\max} - T_{\text{total}}}{T_{\max}}$, where $T_{\max}=10$ is the maximum number of steps allowed per episode, and $T_{\text{total}}$ is the total number of steps taken by the agents. The bonus is to encourage the agents to complete the game as soon as possible. To prevent agents from overfitting to specific scores in the \emph{ScoreG} task while still evaluating in-distribution generalization, we use a held-out set of test scores that do not overlap with those seen during training. Specifically, training scores are randomly sampled from the set $\{5, 10, 15, \dots, 250\}$, while test scores are sampled from the disjoint set $\{2, 4, 6, 8, 12, \dots, 248\}$. This setup ensures that test scores lie within the same distributional range as the training scores, while remaining unseen during training. We also include the game variant to unseen item positions in \autoref{appendix:gen_unseen_pos}. To isolate the role of explicit communication, agents are made invisible to one another by default, i.e., a grid cell occupied by another agent appears empty. This design prevents any form of implicit communication through visual cues \cite{wang2025scout, grupen2022multi, li2021deep, li2024explicit}.

\vspace{-2mm}
\paragraph{Game 2: \emph{Pickup Temporal Order (TemporalG)}} Two agents must cooperatively pick up two items by first navigating to and simultaneously collecting the first spawned item, followed by the second. The pickup must occur in the same order as the items appear.
The environment follows the same configuration as \emph{ScoreG}, except the agent's observation $\mathbf{x}^{(t)}$ excludes the score channel, and the episode length is limited to $T_{\text{max}} = 20$.
Initially, agents are spawned on opposite sides of the grid and remain frozen at time step $t=1$ to $t=6$. Subsequently, each item appears at a distinct time step, drawn uniformly from the fixed duration set $\{1, 2, 3, \ldots, 6\}$. Each item is guaranteed to appear within at least one agent’s receptive field, ensuring it is visible when it spawns. Agents begin moving and attempting to pick up the items only after $t = 6$. To encourage both spatial and temporal displacement, communication is restricted to adjacent grid cells; otherwise, a zero-value message is transmitted. Rewards follow the same scheme as in \textit{ScoreG}.

\section{Experimental Setting}
\vspace{-2mm}
% \MM{very well written overall except for several minors (see bleow). Enjoyable read!}
\label{sec:exp_setting}
%\subsection{Embodied Social Agents}
%\vspace{-2mm}
%\label{sec:embodied_social_agents} 
%Subsequently, we refer to $\mathbf{E}$ as a lookup table of size $\mathbb{R}^{4 \times 16}$, with each entry corresponding to a unique message embedding of dimension $16$.
%\MM{call me to discuss; why E and M? what are the differences?}

%This section provides an overview of how agents interact with their partners and the environment, and how they are trained. A formal description is also given in \autoref{appendix: learning_ppo}. 

\noindent \textbf{Embodied Social Agents.}
As shown in \autoref{fig:agent_architecture}, each agent in our environment learns both to produce actions and to send and receive discrete messages through a policy trained with PPO \cite{schulman2017proximal} (\autoref{appendix: learning_ppo}). For an experiment on unidirectional communication, see \autoref{appendix:unidirect}.
We denote the overall policy of an agent as \(\psi = (\pi, \phi)\), where \(\pi\) is the action policy that selects an environment action, and \(\phi\) is the communication policy that selects a discrete message $\mathbf{m}^{(t)}$ to send.
We formulate communication as a sequence of discrete messages.
At each time step $t$, an agent is only allowed to send one discrete message, where each message is represented by a single integer \( \mathbf{m}^{(t)} \in \{0, 1, 2, 3\} \). This integer indexes into a lookup table of 4 learnable embeddings (each of dimension 16) in the receiving agent. Each agent maintains its own message lookup table, which is not shared. In other words, upon receiving a message \( \mathbf{m}^{(t)} \), the agent retrieves the corresponding message embedding from its own lookup table. We also report results for vocabulary sizes other than the default of 4 in \autoref{tab:vocab_comparison}.

%Both agents' policies are trained jointly and are of the same neural network architecture, as illustrated in \autoref{fig:agent_architecture}.

%, as shown in \autoref{fig:agent_architecture}. 
%They are trained jointly and are parameterized by a single neural network that processes the agent’s input and produces the outputs.

All agents share the same neural network architecture but are initialized with independent random parameters. 
At each time step, the agent receives three inputs: a partial grid observation $\mathbf{x}^{(t)}$, its own position $\mathbf{p}^{(t)}$, and the message $\mathbf{m}^{(t-1)}$ from its partner sent in the previous time step (see \autoref{fig:exchange_message}). Each input is encoded separately, using a multi-layer perceptron for the grid, a linear layer for the position, and a multilayer perceptron for the message. The encoded features are concatenated and passed into a long short-term memory network (LSTM) \cite{hochreiter1997long}, which maintains temporal information in a working memory and outputs a hidden state. This hidden state is used to produce three outputs: a distribution over actions, a distribution over messages, and a scalar value estimate used for computing the PPO advantage \cite{schulman2015high}. Both the action and the message are sampled from their respective probability distributions predicted by the agent. The message is used to index into the receiver agent’s lookup table, retrieving the corresponding embedding, which is then used by the communication policy \(\phi\) of the receiver agent to produce the message output for the current time step $t$.
%\MM{i don't understand this; call me; once I sent you an integer and then so what? how do you retrieve msg from integer?}

%\vspace{-3mm}
%\subsection{Training Regimes}
%\vspace{-2mm}
\noindent \textbf{Training Regimes.}
We train each agent independently using standard single-agent PPO, with no shared parameters or centralized critic. Each agent treats its partner as part of the environment and learns solely from its own experience. Communication policies emerge through interaction and reward, without explicit supervision. This decentralized setup is intentional: it allows agents to specialize, diverge, and adapt to others, enabling us to study how communication protocols evolve under varying population sizes, agent heterogeneity, and connectivity patterns among agents.
To investigate how communication strategies emerge and generalize in populations, we explore two training regimes that differ in how agents interact with other partners. These regimes allow us to study the effects of population diversity, message alignment, and exposure to self-produced messages on communication development.
\noindent \textbf{Cross-Play Training (XP):} We define the agent population size as \( N_{\text{pop}} \). In each training episode, a pair of agents is randomly sampled from the population to play together.
\noindent \textbf{Cross-Play-and-Self-Play Training (XP+SP): } Motivated by this aspect of inner monologue, humans being able to speak to themselves, in each episode, we randomly select either two distinct agents (Cross-Play, XP) or the same agent twice (Self-Play, SP). Specifically, for each episode, we randomly sample agents $i$ and $j$ such that either $i \neq j$ or $i=j$, where $i,j \in \{1,\dots,N_{\text{pop}}\}$.

%\noindent \textbf{Cross-Play Training (XP):} We define the size of an agent population as \( N_{\text{pop}} \). Each agent maintains its own policy. In each training episode, the unique policies of a pair of agents are randomly sampled from the population to play together. 

% \vspace{-2mm}
% \subsection{Social Network Structure}
% \label{sec:social_net}
% \vspace{-2mm}
\noindent \textbf{Social Network Structure.}
We consider several social network structures (\autoref{fig:social_net}): 
Fully Connected (FC), where every agent interacts with all others; Ring, where each agent interacts only with immediate neighbors; Ring with Cliques (Clq), which adds short-range clusters to the ring; Watts–Strogatz (WS), which rewires a fraction of ring edges to create small-world shortcuts; and Ring with Long-Range Connections (LRC), which augments the ring with distant edges. FC represents the dense baseline. Ring is the sparsest structure while still connected, allowing us to test how communication propagates along long paths. Adding cliques or long-range links (Clq, WS, LRC) reduces social distance and is expected to strengthen language transmission compared to Ring ( \autoref{appendix:cultural_tranmission}).

%\revise{Rebuttal Table 1}.
%, to assess if agents understand the messages they produce themselves, i.e., the generalization of language within agents. 
% \vspace{-2mm}
% \subsection{Evaluation Metrics on Language Properties}
% \vspace{-2mm}
% \label{sec:metrics}
\noindent \textbf{Evaluation Metrics on Language Properties.}
To evaluate linguistic properties in emergent communication, we use two standard metrics and a new metric below. %Topographic Similarity \cite{brighton2006understanding, lazaridou2018emergence} and Language Similarity \cite{kim2021emergent, rita2022role, michelrevisiting}. 
%We also introduce a new metric, Interchangeability below. 
Formal definitions of these metrics are provided in \autoref{appendix:metrics}. Briefly, \textbf{Topographic Similarity (\emph{topsim})} \cite{brighton2006understanding, lazaridou2018emergence} measures the structural alignment between message space and semantic space. It is computed as the Spearman correlation between pairwise distances in the message space and the corresponding distances in the semantic space, which reflects the agent’s observations and context. A higher \emph{topsim} indicates a more compositional and consistent mapping between semantic meanings and messages. 
%We define the semantic space using ground-truth attributes of the environment, specifically the scores and positions of items in the \textit{ScoreG} game of FG. To implement the topographic similarity metric, we use the EGG toolkit \cite{kharitonov2019egg}.
\textbf{Language Similarity (\emph{LS})} \cite{kim2021emergent, rita2022role, michelrevisiting} measures how similarly two agents communicate in the same situation. It compares the sequences of discrete messages each agent produces and computes the average agreement between them. 
%A higher score means the agents tend to use the same messages in similar contexts, indicating stronger convergence in their communication strategies.
\textbf{Interchangeability (\emph{IC})} measures whether an agent can understand the language it produces. We define \textbf{Self-SR} as the success rate when an agent plays with a copy of itself, and \textbf{Cross-SR} as the success rate when paired with a different agent.
The overall success rate is denoted as \textbf{SR}. 
\emph{IC} is the ratio of Self-SR to Cross-SR. A high \emph{IC} indicates that agents can better generalize and interpret their own language.
%self-understandable communication. 
%This reflects a fundamental property of language that speakers can comprehend the language they use.
%(See Appendix \autoref{fig:pickup_high_learning_dynamics}).
% \vspace{-2mm}
% \subsection{Implementation Details}
% \vspace{-1mm}

\noindent \textbf{Implementation Details.}
All training can be performed on a single GeForce RTX 4090 (See Appendix \autoref{fig:pickup_high_learning_dynamics} and \autoref{fig:pickup_temp_learning_dynamics}). We use the ADAM optimizer \cite{kingma2014adam} with an initial learning rate of $0.00025$, which linearly decays over time. We report the hyperparameters of the architecture and training algorithm in \autoref{appendix: model_training} (\autoref{tab:neural_arch} and \autoref{tab:ppo_hyperparams}).

%\textbf{Message analysis.}
To examine emergent communication properties, we collect each agent’s messages over all time steps in every episode. For the \emph{ScoreG} task, messages are concatenated into a single message chain per agent per episode. In the \emph{TemporalG} task, agents are only allowed to communicate when they occupy adjacent grid cells. Thus, we begin collecting messages only after their first adjacency and concatenate all subsequent messages. 
To decode environmental attributes, such as item scores, spawn times, and positions, from message chains, we use one-vs-rest logistic regression (LR). For example, to decode four possible item locations, we train four binary classifiers, each distinguishing one location from the rest. During inference, the class with the highest predicted probability is selected.
Each LR model is trained on 3,500 message chains and tested on 1,500, using 3-fold cross-validation across 3 random seeds. We report message chain decoding accuracy as the average across all agents. Other metrics, including \emph{LS}, \emph{IC}, \emph{topsim}, and \emph{SR}, are computed using 1,000 test episodes.

\vspace{-3mm}
\section{Result}
\label{sec:results}
\vspace{-3mm}
\begin{table}[t]
\footnotesize
\centering
\caption{
\textbf{Cross-play training can cause non-interchangeable languages.} Game performance and language similarity (LS;  \autoref{appendix:metrics}) comparison across different training regimes. $N_{\text{pop}}$ is the population size. 
%Values are reported as mean $\pm$ standard deviation. 
\emph{Cross-SR} and \emph{Self-SR} are the mean ($\pm$ standard deviation) success rates when agents play with others or with copies of themselves, respectively.
%of agents played with others. \emph{Self-SR} is the average success rate of agents played with themselves. 
}
\label{tab:method_comparison}
\begin{tabular}{lccccc}
\toprule
\textbf{Training} & $N_{\text{pop}}$ & \textbf{LS} & \textbf{Cross-SR} & \textbf{Self-SR} \\
\midrule
XP & 2 & $0.215 \pm 0.002$ & $0.987 \pm 0.002$ & $0.065 \pm 0.054$ \\
XP+SP & 2 & $0.598 \pm 0.010$ & $0.968 \pm 0.005$ & $0.968 \pm 0.002$ \\
XP & 3 & $0.527 \pm 0.036$ & $0.977 \pm 0.007$ & $0.944 \pm 0.018$ \\
\bottomrule
\end{tabular}
\vspace{-4mm}
\end{table}
%\revise{M2Dx.1, M2Dx.4:  I like to read an objective, quantitative account of the experimental results before I am told how they might be interpreted with respect to the research question. This is a preference which I acknowledge may not be shared by all. A shorter, crisper, more objective Results section would make the paper better.  In addition, the results section occasionally mingles results and analysis, which I found confusing.} \maytus{I don't agree with "mixing results with analyses is a bad thing and confusing". Many scientific papers did that and they read well (See Alpha Fold on Nature for example). The part to be improved is providing motivation and objective of each study before reporting the results.}
We investigate the emergence of linguistic properties in the FG framework according to previously stated research questions (\textbf{Q1-Q5}) and present our findings in this section. Each experimental setting is defined by four factors: the game type, agent population size, social structure, and training regime. These factors influence the emergence of different language properties, making exhaustive analysis infeasible. Instead, we report results from a representative subset of settings.
For clarity, we adopt the naming convention for each experimental setting: \texttt{[game]-[population size]-[social structure]-[training regime]}. For example, \emph{ScoreG}-P15-Ring-XP+SP refers to a setting where $N_{\text{pop}}=15$ agents in a ring-structured network are trained with XP+SP regime on the \emph{ScoreG} game. This convention is used throughout our analysis.

\textbf{Increasing population size and self-play training support interchangeable language [\emph{ScoreG}-FC].}
%\textbf{[Game Setup]} We perform analysis on the agents in \emph{TemporalG}-Ring-XP+SP. 
%\revise{M2Dx.1, M2Dx.4: Motivation and objective for the experiment.} 
Since decentralized agents do not share neural network parameters, it is natural to wonder whether they develop a common language at all. A curious possibility is that one agent might produce L1 while only understanding L2, with the other doing the opposite. Moreover, it is possible that an agent trained with its partner does not understand the language it produces.
As shown in \autoref{tab:method_comparison}, XP agents trained exclusively with each other fail to understand their own messages when paired with a copy of themselves during test time as indicated by a drop of 6\% in \emph{Self-SR}.
Surprisingly, when agents are trained with
$N_{\text{pop}} = 3$, they maintain high \emph{Self-SR} even without explicit self-play, suggesting that population training encourages the emergence of interchangeable language. These agents also develop more consistent language, as indicated by a higher \emph{LS} of 0.53 compared to the XP alone. Furthermore, XP+SP agents achieve high \emph{Self-SR} due to their direct exposure to SP. Interestingly, they also exhibit greater \emph{LS} across agents, with an \emph{LS} of 0.59, outperforming XP agents even in population settings. Agents with high \emph{LS} produce similar messages, while agents with low \emph{LS} produce distinct ones in the embedding space (\autoref{fig:embedding_by_agents}). Finally, we provide a theoretical analysis showing why training with more partners and with SP leads to more common language in \autoref{sec:theory}.

\textbf{Population size affects compositionality [\emph{ScoreG}-FC].}
%\revise{M2Dx.1, M2Dx.4: Motivation and objective for the experiment.} 
Linguistic and cognitive studies suggest that larger communities are more likely to develop compositional languages \cite{reali2018simpler, raviv2019larger}. Motivated by this, we investigate how compositionality varies with increasing population size.
As shown in \autoref{fig:topsim_vs_pop}, population size has an effect on \emph{topsim}. For XP agents, \emph{topsim} increases from population size 2 and saturates from size 6. In contrast, XP+SP agents show a more consistent upward trend in \emph{topsim} as population size increases, plateauing around size 12–15. Interestingly, XP+SP agents trained in a population of size 2 also achieve the highest \emph{topsim} across all population sizes. However, we caution against the interpretation of effective language from such a small population because this language could be overfitting to a fixed partner and may not generalize well to more diverse partners. While population training promotes consistency, we observe that compositionality does not increase monotonically, aligned with results reported in previous studies \cite{kim2021emergent, rita2022role, michelrevisiting}.

\begin{figure}[t]
    \centering
    \begin{subfigure}[b]{0.24\textwidth}
        \centering
        \includegraphics[trim=10 20 10 10, clip, width=\textwidth]{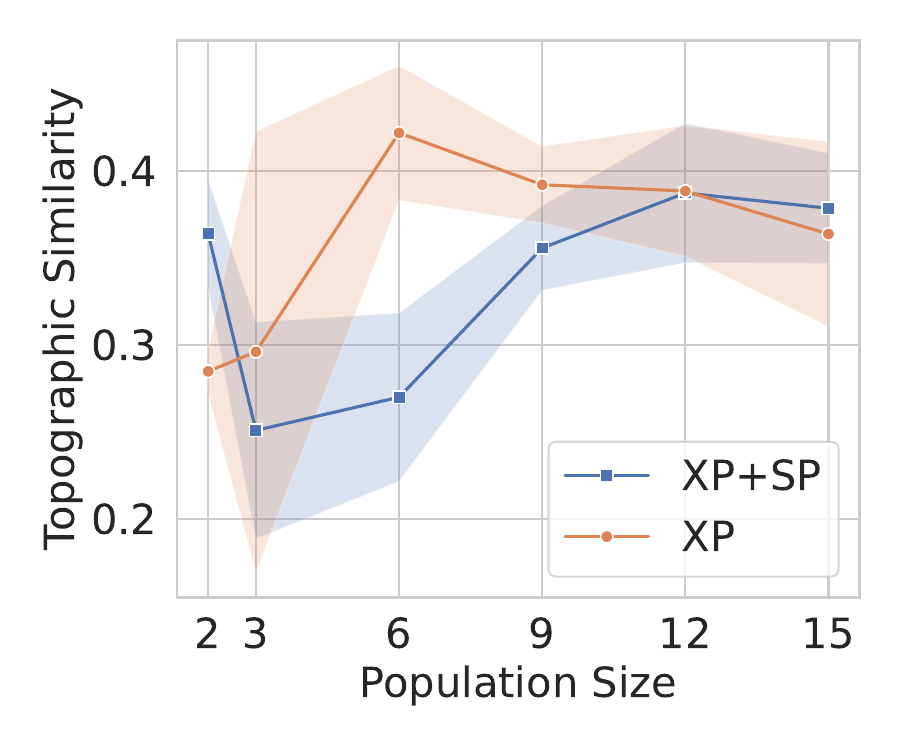}
        \caption{\emph{topsim} across $N_\text{pop}$}
        \label{fig:topsim_vs_pop}
    \end{subfigure}
    % \hfill
    \begin{subfigure}[b]{0.24\textwidth}
        \centering
        \includegraphics[trim=10 20 10 10, clip, width=\textwidth]{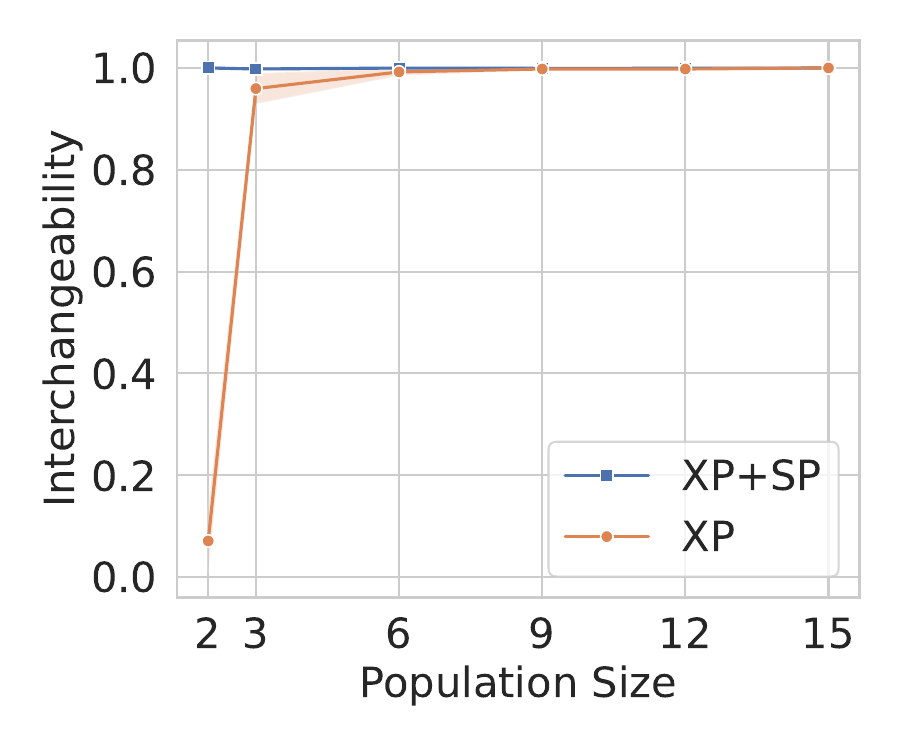}
        \caption{\emph{IC} across $N_\text{pop}$}
        \label{fig:ic_vs_pop}
    \end{subfigure}
    % \hfill
    \begin{subfigure}[b]{0.24\textwidth}
        \centering
        \includegraphics[trim=10 20 10 10, clip, width=\textwidth]{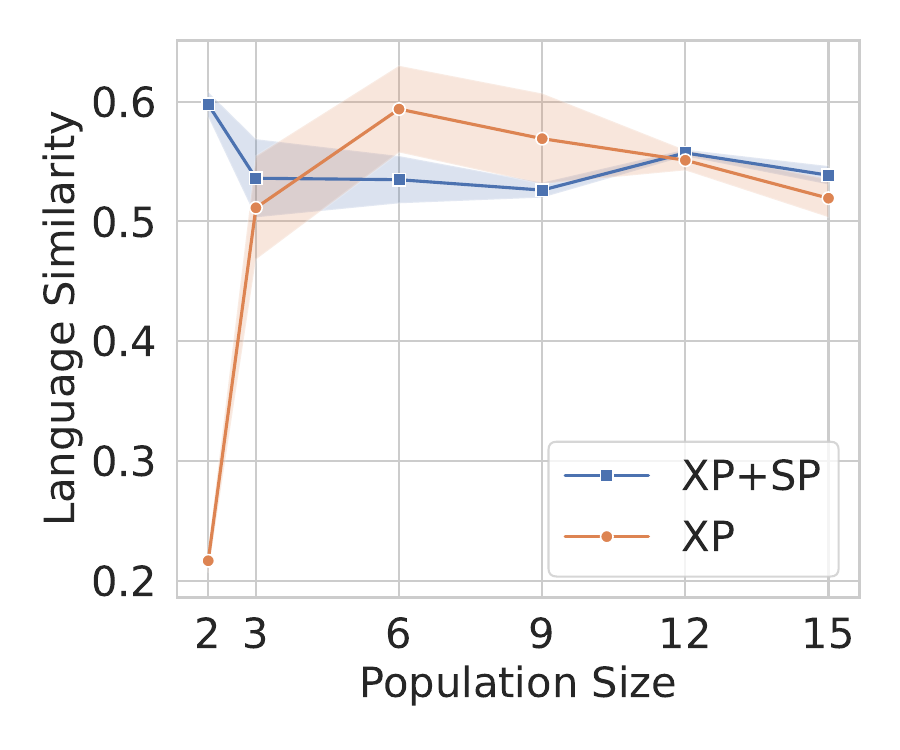}
        \caption{\emph{LS} across $N_\text{pop}$}
        \label{fig:ls_vs_pop}
    \end{subfigure}
    \vspace{-2mm}
    \caption{\textbf{The effect of population training with fully-connected social networks on language properties.} (a) Topographic Similarity (\emph{topsim}), (b) Interchangeability (\emph{IC}), and (c) Language Similarity (\emph{LS}) as a function of population sizes are shown. See \autoref{appendix:metrics} for metrics details.
    }
    \label{fig:comparison}
    \vspace{-7mm}
\end{figure}

\textbf{Population size affects language similarity and interchangeability [\emph{ScoreG}-FC].}
%\revise{M2Dx.1, M2Dx.4: Motivation and objective for the experiment.} 
We observed that a population of three agents can develop a shared, interchangeable language. We now ask whether this effect persists as the population size increases.
\autoref{fig:ic_vs_pop} shows that IC quickly saturates to nearly 1.0 for XP agents once the population size exceeds 2. This indicates that agents can understand their own messages. For XP+SP agents, IC reaches the maximum level across $N_{\text{pop}}\geq
2$. As shown in \autoref{fig:ls_vs_pop}, for XP agents, LS increases sharply from size 2 to 3, then gradually plateaus, indicating that training with multiple partners encourages the emergence of a shared protocol. XP+SP agents, in contrast, show consistently high LS across all population sizes. This suggests that self-play helps promote consistent, shared communication protocols within a population. Together, these results suggest that both self-play and population diversity contribute to the emergence of interchangeable and shared communication. Unlike our findings, prior work \cite{kim2021emergent} reported that larger population sizes reduce language similarity, as listeners adapt to multiple speaker-specific languages. This discrepancy may stem from their unidirectional communication and centralized training, resulting in different optimal solutions.

%their unidirectional communication setup and shared gradients through the communication channel likely result in different optimal solutions than our bidirectional, decentralized (no shared gradients) setting.

\textbf{Self-play enhances cultural transmission in weakly-connected networks [\emph{ScoreG}-P15].} 
%\revise{Rewrite this part M2Dx.1, M2Dx.4: Motivation and objective for the experiment.} 
Human language evolves in sparsely connected societies, giving rise to a shared language alongside dialects and foreign languages. To investigate how language propagates in such settings, we train agents within weakly connected Ring and Clq social networks. The population size is fixed at $N_{\text{pop}} = 15$, with each agent interacting only with a few neighbors. This setup allows us to examine how culture is transmitted across indirect connections at test time. As shown in \autoref{fig:ring_ls_sr}, both \emph{LS} and \emph{SR} generally decline with increasing distance between agents. This indicates that language partially propagates through non-co-trained agents: nearby neighbors develop more similar languages than distant ones. Importantly, the learned language also generalizes to indirectly connected agents, albeit with reduced \emph{SR}. An exception arises for XP agents under the Ring network, which display a distinctive zig-zag pattern in \autoref{fig:ls_vs_dis} and \autoref{fig:sr_vs_dis}. Here, an agent produces messages more similar to those of a non-neighbor than to its direct partner (\autoref{fig:ring_xp_ls}). While agents successfully learn to interpret their neighbors, they fail to produce messages aligned with the language they understand (\autoref{fig:ring_xp_ls} and \autoref{fig:ring_xp_sr}), highlighting a breakdown in interchangeability. Furthermore, as shown in \autoref{fig:ls_vs_dis} and \autoref{fig:clq_sm_ls_vs_dis}, self-play training yields consistently higher \emph{LS} across all social distances compared to cross-play training. Together, these findings suggest that self-play not only enhances within-agent consistency but also supports more stable and transmissible communication protocols in structured populations. Finally, we find that adding more connections to Ring can greatly enhance cultural transmission (\autoref{fig:sm_ls_sr}).

\textbf{Spatial and temporal displacement emerge in agent communication [P3-FC-XP].} Humans can refer to past or future events, for example, I saw an apple yesterday or I will collect an apple tomorrow. They can also refer to spatially distant objects, such as an apple in the forest. This ability to refer to things beyond the immediate here and now is known as displacement in natural language. We investigate whether displacement can emerge in artificial foraging agents using the \emph{TemporalG} game.
XP agents with $N_{\text{pop}}=3$ achieve strong performance in this setting, with a Cross-SR of $0.975$ and a Self-SR of $0.962$ (\autoref{tab:temporal_order_performance}). Since the communication range is limited, the agents adopt a rendezvous strategy, meeting near the center of the grid map to establish communication (\autoref{fig:rendezvous}). Next, we decode spatial and temporal displacement information from the messages. The items can take on one of 6 possible spawn times, 4 vertical positions (excluding the grid's center), and 5 possible horizontal positions.
%, such as 4 vertical positions excluding the middle of the grid and 5 horizontal positions of the items. 
The decoding accuracy for the spawn time and position of items is above chance (\autoref{fig:decode_temp}), suggesting that messages function as temporal adverbials and spatial references. 
We also decode items' scores and positions from XP agents ($N_{\text{pop}}=3$) in the \emph{ScoreG} game (\autoref{fig:decode_high}). The items can take on one of 10 possible score ranges, 2 possible vertical positions (top and bottom), and 5 possible horizontal positions.
The results align with those observed in the \emph{TemporalG} game. Moreover, 
compared to \emph{Integer Msg}, decoding with chains of message embeddings yields higher accuracy, suggesting that these embeddings encode more meaningful and linearly separable spatial and temporal features. Interestingly, when agents collect targets remotely without movement, their messages stop conveying spatial information (\autoref{fig:decode_rg}), highlighting the role of embodiment in shaping language.

\begin{figure}[t]
    \centering
    \begin{subfigure}[b]{0.24\textwidth}
        \centering
        \includegraphics[trim=10 20 10 10, clip, width=\textwidth]{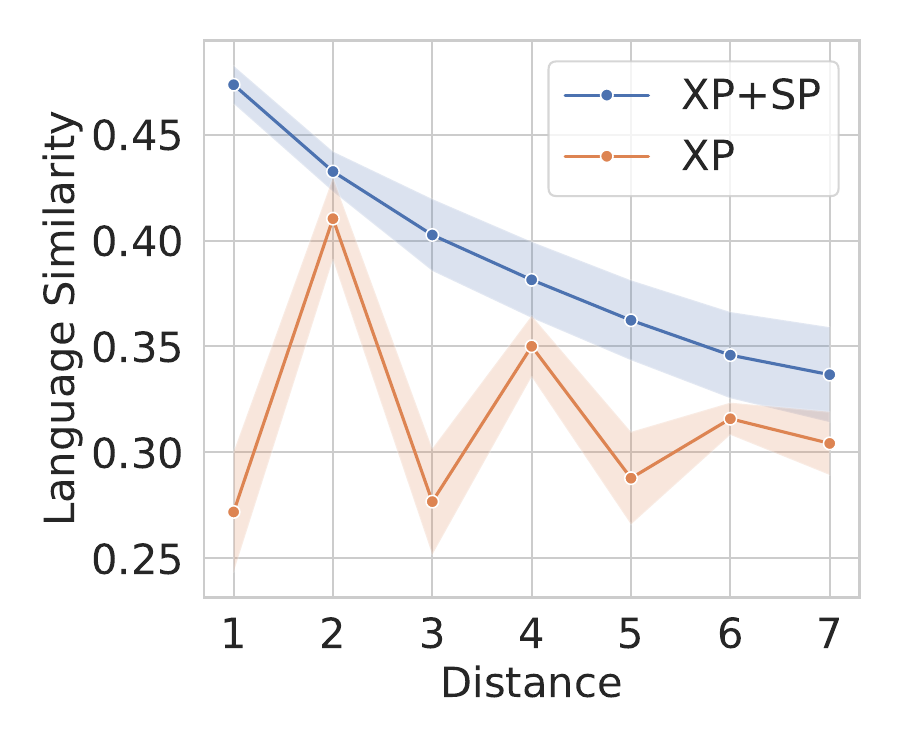}
        \caption{\emph{LS}}
        \label{fig:ls_vs_dis}
    \end{subfigure}
    \begin{subfigure}[b]{0.24\textwidth}
        \centering
        \includegraphics[trim=10 20 10 10, clip, width=\textwidth]{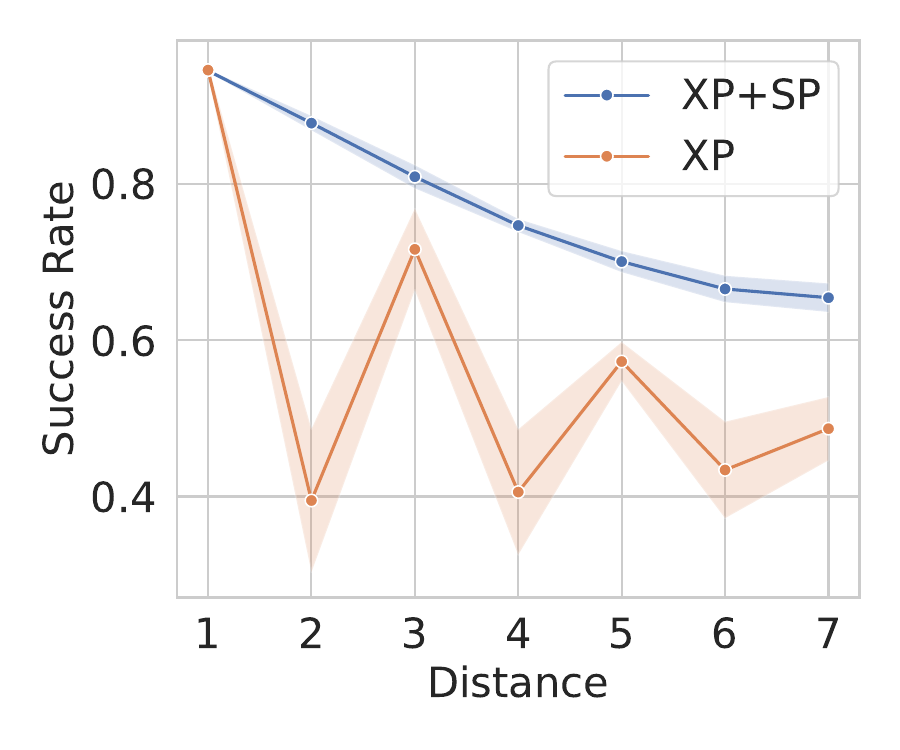}
        \caption{\emph{SR}}
        \label{fig:sr_vs_dis}
    \end{subfigure}
    \begin{subfigure}[b]{0.24\textwidth}
        \centering
        \includegraphics[trim=10 20 10 10, clip, width=\textwidth]{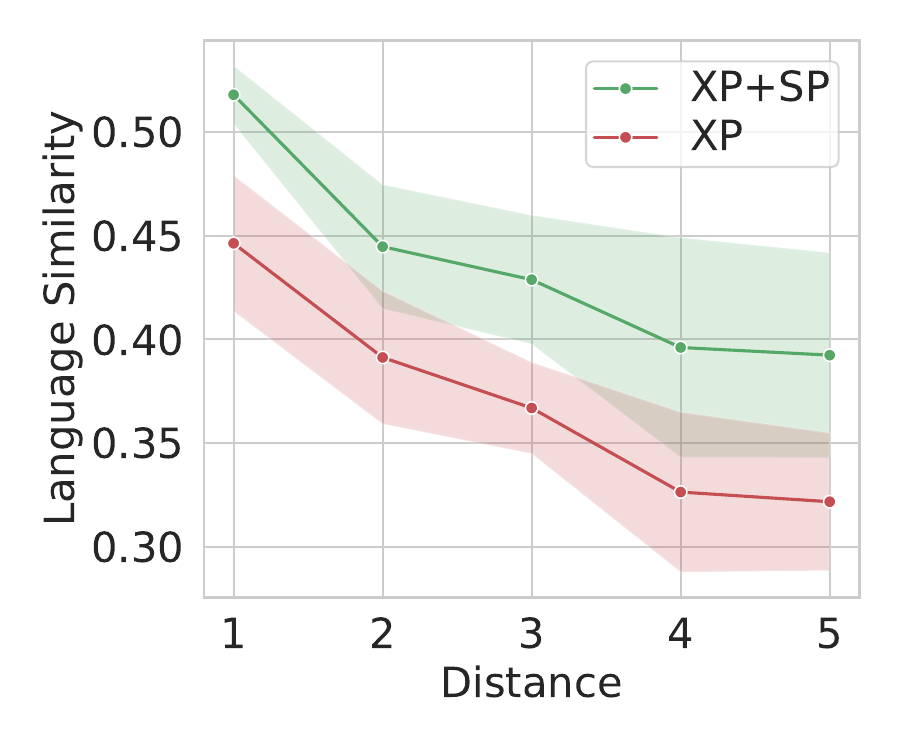}
        \caption{\emph{LS}}
        \label{fig:clq_sm_ls_vs_dis}
    \end{subfigure}
    \begin{subfigure}[b]{0.24\textwidth}
        \centering
        \includegraphics[trim=10 20 10 10, clip, width=\textwidth]{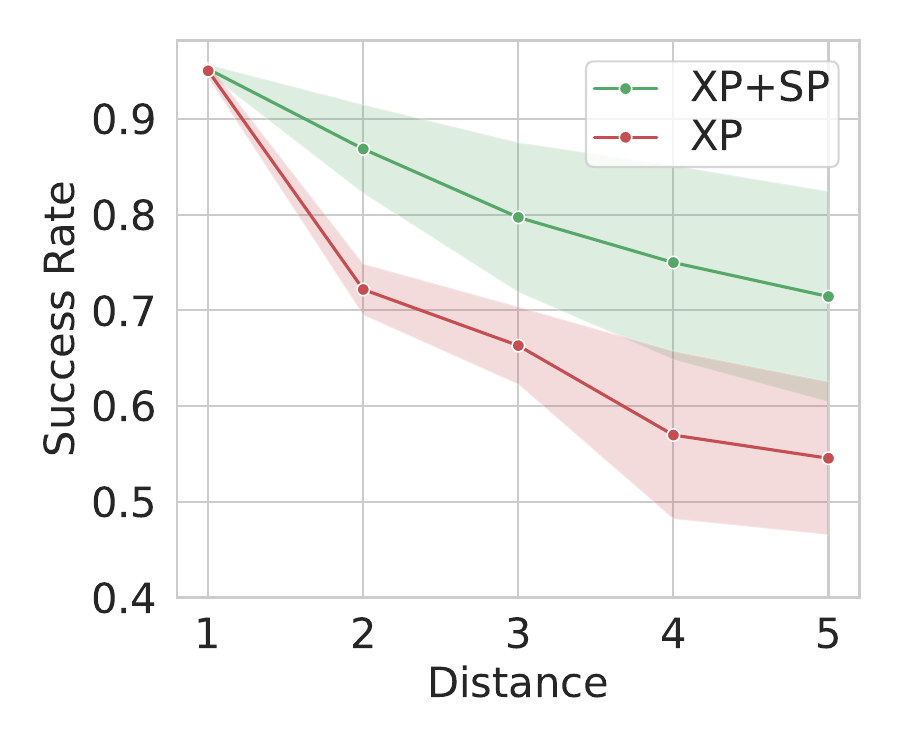}
        \caption{\emph{SR}}
        \label{fig:clq_sm_sr_vs_dis}
    \end{subfigure}
    \vspace{-2mm}
    \caption{\textbf{Self-Play (SP) promotes cultural transmission across social networks ($N_{\text{pop}}=15$):} We plot \emph{LS} and \emph{SR} (see \autoref{appendix:metrics}) as a function of the shortest-path distance between two agents. (a,b) ring-structured network (Ring). (c,d) ring-structured with clique network (Clq). 
    }
    \label{fig:ring_ls_sr}
    \vspace{-2mm}
\end{figure}

\textbf{Implicit communication can emerge when explicit message communication is disabled [\emph{ScoreG}-P3-FC-XP].} We study whether agents can convey information through their sequence of actions in an episode when the explicit communication channel is disabled. We conduct an ablation on two variables: partner visibility and the presence of explicit verbal communication.
When a partner is invisible, their location appears as an unoccupied cell in the observing agent's receptive field. We train XP agents in \emph{Inv-NoCom} and \emph{Vis-NoCom}, where agents either have or lack partner visibility, but no explicit message communication is allowed. 
%\autoref{fig:ablation_all} for the description and results. 
If an agent observes an extreme score (e.g., 222 or 22), it can confidently infer whether the item is a goal or not. To isolate this confounding factor, we evaluate agents only under the condition that both partners observe high scores. As shown in \autoref{fig:ablation_all}, \emph{Inv-Com} agents (XP agents trained in our default FG environments) achieve SR of $85\%$. Without communication, \emph{Inv-NoCom} agents perform below chance ($40\%$), likely because they fail to coordinate goal pickup without seeing each other. See more failure analysis in \autoref{appendix:implicit-sr}.
Interestingly, \emph{Vis-NoCom} agents achieve an SR of 60\%, outperforming \emph{Inv-NoCom}. This indicates that both \emph{Inv-Com} and \emph{Vis-NoCom} agents can glean information from their partners—either via explicit messages or by observing partners’ actions. We also evaluate the average length of successful test episodes. Both \emph{Inv-Com} and \emph{Vis-NoCom} agents take around 6 steps, while \emph{Inv-NoCom} agents take about 5 steps. This indicates that agents with more accurate communication require extra steps to send meaningful signals, increasing their bandwidth and hence, longer episode length. Another possible strategy for implicit communication is bumping into each other; however, our analysis in \autoref{appendix:scoreg_bumping} shows this is not the case. 
The similar analysis above on implicit communication in \emph{ScoreG} is also applicable in \emph{TemporalG} (\autoref{appendix:implicit_temporalg}).

\begin{figure}[t]
  \centering

  %--- Left minipage: Message Analysis ---
  \begin{minipage}[t]{0.48\textwidth}
    \centering
    \begin{subfigure}[b]{0.48\textwidth}
        \centering
        \includegraphics[trim=30 20 5 5, clip, width=\textwidth]{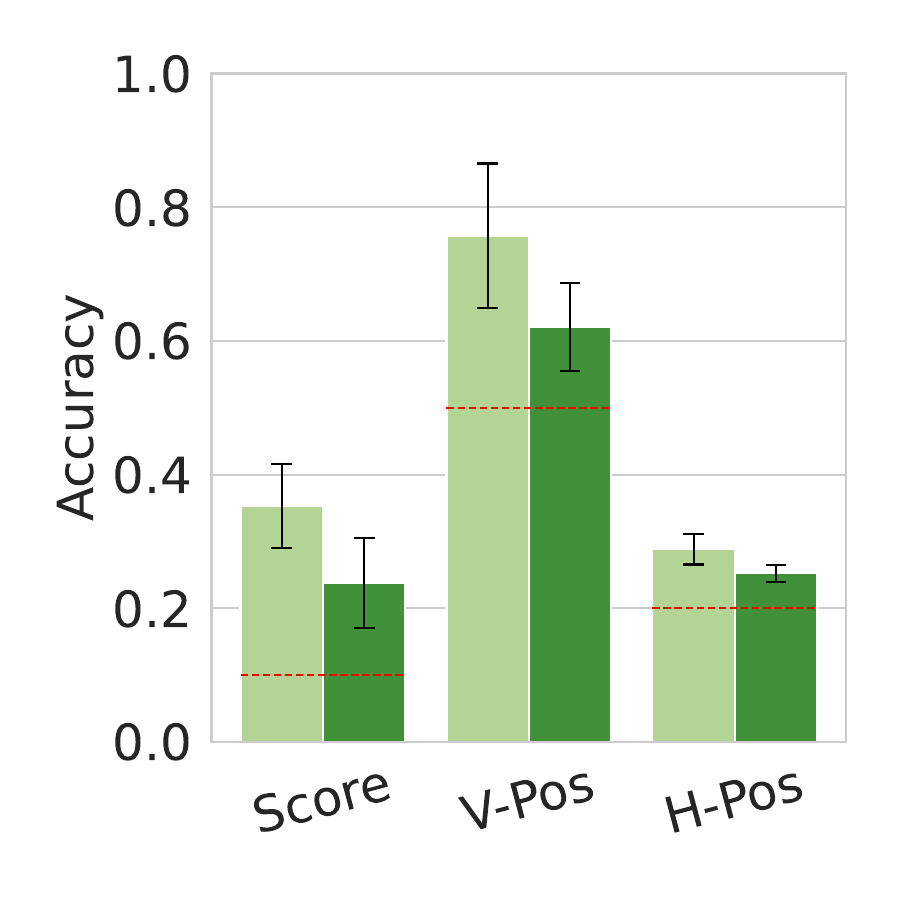}
        \vspace{-4mm}
        \caption{\emph{ScoreG}}
        \label{fig:decode_high}
    \end{subfigure}
    % \hfill
    \begin{subfigure}[b]{0.48\textwidth}
        \centering
        \includegraphics[trim=30 20 5 5, clip, width=\textwidth]{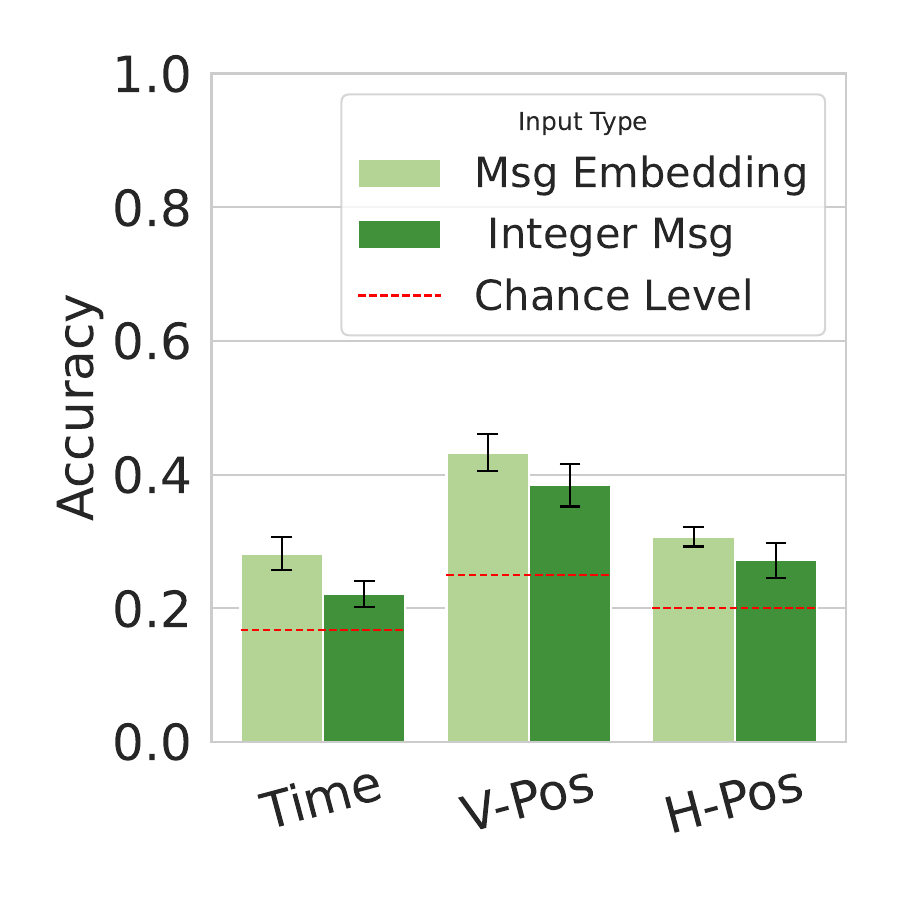}
        \vspace{-4mm}
        \caption{\emph{TemporalG}}
        \label{fig:decode_temp}
    \end{subfigure}
    \vspace{-2mm}
    \caption{
      \textbf{Decoding item states from messages.}
      \emph{V-Pos}/\emph{H-Pos} are item positions, \emph{Score} is item value, and \emph{Time} is item spawn time. \emph{Integer Msg} uses raw message chains composed of integer sequences. \emph{Msg Embedding} uses chained embeddings mapped from a lookup table. Error bars show standard deviations.
    }
    \label{fig:decode_all}
  \end{minipage}
  \hfill
  %--- Right minipage: Ablation Study ---
  \begin{minipage}[t]{0.48\textwidth}
    \centering
    \begin{subfigure}[b]{0.48\textwidth}
        \centering
        \includegraphics[trim=5 10 5 5, clip, width=\textwidth]{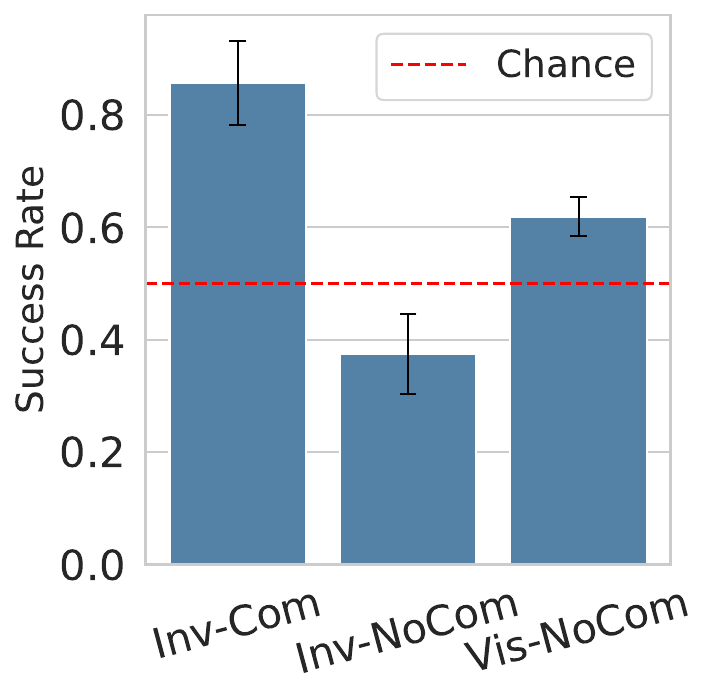}
        \vspace{-4mm}
        \caption{Success Rate}
        \label{fig:ablation_sr}
    \end{subfigure}
    \hfill
    \begin{subfigure}[b]{0.48\textwidth}
        \centering
        \includegraphics[trim=5 10 5 5, clip, width=\textwidth]{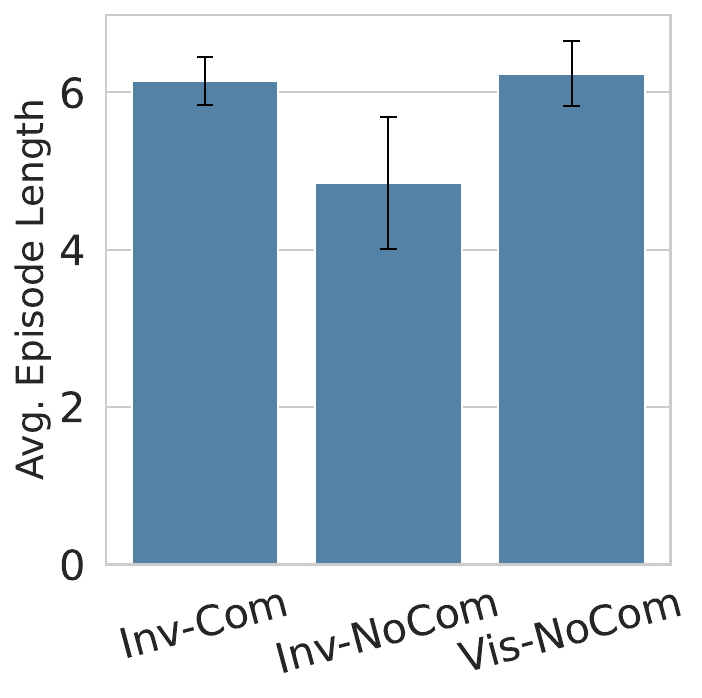}
        \vspace{-4mm}
        \caption{Avg Ep Length}
        \label{fig:ablation_ep_length}
    \end{subfigure}
    \vspace{-2mm}
    \caption{
      \textbf{Implicit communication in the \emph{ScoreG} games.} 
      %is used. 
      \emph{Inv}/\emph{Vis} denote partner visibility; \emph{Com}/\emph{NoCom} indicate presence or absence of verbal communication. Item scores are from $[160, 162, \dots, 240]$. We report average successful episode length; error bars show standard deviations across seeds and agents.
    }
    \label{fig:ablation_all}
  \end{minipage}
  \vspace{-5mm}
\end{figure}

\section{Discussion}
\vspace{-2mm}

Understanding the origins of language is a profound and long-standing challenge that requires insights from linguistics, psychology, neuroscience, and artificial intelligence. Our work represents an initial step toward addressing this complex question. We investigate the language emergence in multi-agent Foraging Games. Through systematic experiments on computational models in simulated environments, we gain valuable insights into key questions in evolutionary linguistics—such as what drives the evolution of language and how it becomes essential for teamwork. To further probe this point, we conducted a non-cooperative control experiment (\autoref{appendix:no_cooperation}), which showed that when agents pursue independent goals, communication provides no advantage and does not emerge meaningfully. This reinforces cooperation as a key ecological pressure in the evolution of language. Remarkably, without any direct supervision from human languages, agents with limited capacity trained in cooperative settings develop five hallmark properties of human language \cite{hockett1960origin}.
\textbf{Interchangeability} emerges through self-play training or social interactions within larger populations
(\autoref{tab:method_comparison}).
%, indicating that agents can understand their own messages. 
\textbf{Arbitrariness} is demonstrated by two cross-play agents developing non-interchangeable languages that are nonetheless mutually intelligible enough to enable successful cooperation and a high task success rate (\autoref{tab:method_comparison} and \autoref{fig:embedding_2xp}).
%\textbf{Arbitrariness} is reflected in the fact that agents develop distinct communication protocols. This is demonstrated by non-interchangeable but mutually intelligible languages developed by two cross-play agents (\autoref{tab:method_comparison} and \autoref{fig:embedding_2xp}). 
%emerges in the structured relationship between messages and input semantic. This
\textbf{Compositionality} is indicated by high \emph{topsim} scores in larger populations, a common metric measuring the structural alignment between messages and environmental semantics (\autoref{fig:topsim_vs_pop}).
%, and successful linear decoding of item attributes like score and position from messages (\autoref{fig:decode_high}), indicating that discrete components of meaning are encoded. 
% in the \emph{TemporalG} game, 
\textbf{Displacement} is demonstrated by high decoding accuracy of items’ positions, scores, and spawn times, showing that agents can refer to when and where currently invisible events happen (\autoref{fig:decode_high}). 
\textbf{Cultural transmission} is evidenced by the gradual decline in LS and SR with increasing distance in structured populations, despite no direct training between distant pairs (\autoref{fig:ring_ls_sr}, \autoref{fig:ring_xpsp_ls}, \autoref{fig:ring_xpsp_sr}). This suggests that culture is transmitted and shaped through interaction within structured social networks. Our framework also shows that communication can emerge in both implicit and explicit forms, which standard referential games do not support.

%\textbf{Limitations:} 
%\revise{ViJF.13: The limitations section is very brief. and VKge.2: Linking the experimental observations to different phenomena in human language.}
% \cite{foerster2016learning, jain2019two, patel2021interpretation}.
%Bridging this gap may require loosening some of the constraints currently imposed, for example, solving more complex tasks may require agents to share gradients or parameters to stabilize the training.
%\cite{sun2023trust, de2020independent, yu2022surprising}. 

Despite these contributions, our work has several limitations. The language emerging in FG is not directly comparable to natural language, remaining simple and lacking syntax and morphology. FG also offers limited semantic richness, as agents primarily communicate about a narrow set of task variables (time, location, score). Progress toward more human-like emergent language will require semantically richer, open-ended environments that encourage broader, compositional, and syntactically structured communication. Finally, our framework does not model turn-taking dynamics, a further challenge for current approaches.

\bibliographystyle{plain}
\bibliography{ref}

\newpage
\renewcommand{\thesection}{S\arabic{section}}
\renewcommand{\thefigure}{S\arabic{figure}}
\renewcommand{\thetable}{S\arabic{table}}
\setcounter{figure}{0}
\setcounter{section}{0}
\setcounter{table}{0}
\appendix

\section{Implementation and Training Details}

\subsection{Learning to Communicate with PPO}
\label{appendix: learning_ppo}
Agents in our setting learn to act and communicate simultaneously. As illustrated in \autoref{fig:agent_comm}, each agent receives an observation from the environment and a message from its partner sent at the previous time step. Based on this input, the agent produces a discrete action and message. During training, the agent also estimates a value function to guide learning. In the following, we formalize this process using the multi-agent reinforcement learning (MARL) framework and describe how action and communication are jointly optimized via the MARL algorithm. The method most closely related to ours is RIAL \cite{foerster2016learning}. However, unlike RIAL and related approaches, we do not assume parameter sharing or use a centralized critic. While Independent PPO \cite{de2020independent} employs parameter sharing among independent agents, a technique that can help stabilize learning \cite{sun2023trust}, our method explicitly removes this parameter sharing. These design choices make our setting more reflective of human-like conditions, where individuals operate independently and learn through decentralized interaction. Moreover, by avoiding parameter sharing, our method enables the study of linguistic and behavioral heterogeneity across agents.

\paragraph{Problem Formulation} Our learning setting is a two-player variant of decentralized partially observable Markov decision processes (DecPOMDP) \cite{bernstein2002_decpomdp}. A two-player DecPOMDP with communication is formally defined as the tuple $ (S, \mathcal{A}_1, \mathcal{A}_2, \mathcal{M}_1, \mathcal{M}_2, \Omega_1, \Omega_2, T, O, r, \gamma, H), $
where $S$ represents the state space, $\mathcal{A} \equiv \mathcal{A}_1 \times \mathcal{A}_2$ denotes the joint-action space, $\mathcal{M} \equiv \mathcal{M}_1 \times \mathcal{M}_2$ denotes the joint-message space, $\Omega \equiv \Omega_1 \times \Omega_2$ signifies the joint-observation space (where $\Omega_1$ and $\Omega_2$ are the individual observation spaces), $T(s'|s, a_1, a_2)$ provides the transition probability from state $s$ to $s'$ upon executing the joint action $(a_1, a_2)$, $O(o_1, o_2|s)$ represents the conditional probability of observing the joint observation $(o_1, o_2)$ given the current state $s$, $r(s, a_1, a_2)$ is the shared reward function, $\gamma$ is the discount factor for rewards, and $H$ denotes the horizon length. 

The first and second agents are controlled by neural policies $\psi_1 = (\pi_1, \phi_1)$ and $\psi_2 = (\pi_2, \phi_2)$, respectively, where the subscript indicates the agent identity. Here, $\pi$ denotes the action policy and $\phi$ denotes the communication policy. Each neural policy $\psi$ is parameterized by $\theta$. At each time step $t$, agents receive the joint observation $o^{(t)} = (o^{(t)}_1, o^{(t)}_2)$ and the previous joint message $m^{(t-1)} = (m^{(t-1)}_1, m^{(t-1)}_2)$. Each agent $i$ maintains a trajectory $\tau^{(t)}_i$ consisting of its own history up to time $t$. Conditioned on its trajectory $\tau^{(t)}_i$ and the other agent’s previous message $m^{(t-1)}_j$, agent $i$ samples its action and message as follows: $a^{(t)}_i \sim \pi_i(a_i^{(t)} \mid m^{(t-1)}_j, \tau^{(t)}_i)$ and $m^{(t)}_i \sim \phi_i(m_i^{(t)} \mid m^{(t-1)}_j, \tau^{(t)}_i)$. The joint trajectory can be written as $\tau = (o_0,a_0,m_0,r_1,o_1,...,o_{H-1}, a_{H-1}, m_{H-1}, r_{H},o_{H}) \in \mathcal{T} \equiv (\Omega \times \mathcal{A} \times \mathcal{M} \times \mathbb{R})^H$. The return of a trajectory $\tau$ is defined as $G(\tau) = \sum_{t=1}^{H} \gamma^{t-1} r_t$, and the expected return of the joint policy $(\psi_1, \psi_2)$ is $J(\psi_1,\psi_2) = \mathbb{E}_{\tau \sim p(\tau | \theta_1, \theta_2)}G(\tau)$, where $p(\tau | \theta_1, \theta_2)$ denotes the distribution over trajectories induced by the joint policy.

\paragraph{Learning to Act and Communicate with PPO} As described above, we treat the communication policy the same as the action policy, i.e., we jointly optimize them with policy gradients. This is because both policies output discrete decisions based on the agent’s observation history, and both can be optimized with similar objective functions under policy gradient methods. We use PPO \cite{schulman2017proximal} because of its effectiveness in optimizing discrete actions, sample efficiency, and good performance in multi-agent learning \cite{li2023ace, de2020independent, yu2022surprising}. We first define $r_a^{(t)}(\theta) = \frac{\pi(a^{(t)}|\tau^{(t)})}{\pi_{\text{old}}(a^{(t)}|\tau^{(t)})}$ and $r_m^{(t)}(\theta) = \frac{\phi(m^{(t)}|\tau^{(t)})}{\phi_{\text{old}}(m^{(t)}|\tau^{(t)})}$ as the ratios of action policy and communication policy \cite{schulman2017proximal}, respectively. The objective function of the neural policy $\theta$ can be written as $\mathcal{J} = \mathcal{J}_{a} + \mathcal{J}_{m} + \mathcal{J}_{\text{ent}}$. We use Generalized Advantage Estimation (GAE) \cite{schulman2015high} to estimate the advantage function. We use a single advantage function $\hat{A}^{(t)}$ estimated from the shared representation output by the LSTM, which is used for optimizing both the action and communication policies. Each agent is trained independently using standard single-agent PPO, which treats its partner as part of the environment, i.e., fully decentralized training.
The first term optimizes the action policy $\pi$ using the estimated advantage function:
\begin{equation}
 \mathcal{J}_{a} = \mathbb{E}_{t, \tau \sim p(\tau | \theta, \cdot)}\left[\min\left( r_a^{(t)}(\theta) \hat{A}^{(t)},\; \text{clip}(r_a^{(t)}(\theta), 1 - \epsilon, 1 + \epsilon) \hat{A}^{(t)} \right)\right].
\end{equation}
The second term optimizes the communication policy $\phi$ in the same manner:
\begin{equation}
    \mathcal{J}_{m} = \mathbb{E}_{t, \tau \sim p(\tau | \theta, \cdot)}\left[\min\left( r_m^{(t)}(\theta) \hat{A}^{(t)},\; \text{clip}(r_m^{(t)}(\theta), 1 - \epsilon, 1 + \epsilon) \hat{A}^{(t)} \right)\right].
\end{equation}
The final term encourages exploration by maximizing the entropy of both action and message distributions:
\begin{equation}
    \mathcal{J}_{\text{ent}} = -\mathbb{E}_{t, \tau \sim p(\tau | \theta, \cdot)}\left[ 
    \lambda_a \sum_{a} \pi(a|\tau^{(t)}) \log \pi(a|\tau^{(t)}) + 
    \lambda_m \sum_{m} \phi(m|\tau^{(t)}) \log \phi(m|\tau^{(t)}) \right].
\end{equation}

\paragraph{Neural Architecture} In \autoref{fig:agent_comm}, the agent architecture is divided into input and output components. The encoders are essential for converting heterogeneous inputs such as messages, grid observations, and spatial coordinates into a unified feature representation, enabling effective downstream processing and integration by the LSTM. On the input side, the \emph{message encoder} $\mathcal{E}_\mathcal{M}$, \emph{grid encoder} $\mathcal{E}_\mathcal{X}$, and \emph{position encoder} $\mathcal{E}_\mathcal{P}$ process the incoming message from the other agent, the grid observation, and the agent’s current position, respectively. The extracted features from these encoders are concatenated and fed into an LSTM~\cite{hochreiter1997long}, which maintains temporal memory. On the output side, the \emph{message head} $\mathcal{D}_\mathcal{M}$, \emph{action head} $\mathcal{D}_\mathcal{A}$, and \emph{value head} $\mathcal{D}_\mathcal{V}$ generate the outgoing message, the selected action, and the estimated state value, respectively. All encoders and heads are implemented as shallow multilayer perceptrons (MLPs). The message encoder $\mathcal{E}_\mathcal{M}$ includes an embedding layer that maps each vocabulary index to a real-valued vector before passing it to the MLP.

\subsection{Model Training}
\label{appendix: model_training}
We train our models on a single GeForce RTX 4090. However, due to the small model size and low environment complexity, training can also be run entirely on a CPU as well. For populations smaller than 3, training takes approximately 16 hours, depending on the CPU and the number of available cores. The longest training time occurs with a population size of 15 and takes up to 3 days. All experiments are repeated with three different random seeds. We use the ADAM optimizer \cite{kingma2014adam} with an initial learning rate of $0.00025$, which linearly decays over time. We report the hyperparameters of the architecture and training algorithm in \autoref{tab:neural_arch} and \autoref{tab:ppo_hyperparams}. The examples of learning dynamics for \emph{ScoreG} and \emph{TemporalG} are shown in \autoref{fig:pickup_high_learning_dynamics} and \autoref{fig:pickup_temp_learning_dynamics}.

\begin{table}[ht]
\centering
\caption{Neural architecture hyperparameters. \\}
\label{tab:neural_arch}
\begin{tabular}{ll}
\toprule
\textbf{Component} & \textbf{Hyperparameter} \\
\midrule
Visual encoder & 4-layer MLP: [256, 256, 128, 16] \\
Visual input shape & $(1, 5, 5)$ grayscale grid \\
Position encoder & Linear(2, 4) \\
Message embedding size & 16 \\
Message encoder & Embedding(Vocab. Size, 16) + Linear(16, 16) \\
LSTM hidden size & 128 \\
LSTM input size & $16$ (visual) $+$ $4$ (position) $+$ $16$ (message) = 36 \\
Action head & Linear(128, \texttt{num\_actions}) \\
Value head & Linear(128, 1) \\
Message head & Linear(128, 16) \\
Weight initialization & Orthogonal (std = $\sqrt{2}$), output layers: std = 0.01 \\
\bottomrule
\end{tabular}
\end{table}

\begin{table}[ht]
\centering
\caption{PPO hyperparameters. \\}
\label{tab:ppo_hyperparams}
\begin{tabular}{ll}
\toprule
\textbf{Hyperparameter} & \textbf{Value} \\
\midrule
Total time steps & $2 \times 10^9$ \\
Learning rate & $2.5 \times 10^{-4}$ \\
Number of environments & 128 \\
Number of steps per rollout & 32 \\
Number of minibatches & 4 \\
Number of update epochs & 4 \\
Discount factor $\gamma$ & 0.99 \\
GAE $\lambda$ & 0.95 \\
Clip coefficient & 0.1 \\
Clip value loss & True \\
Normalize advantages & True \\
Action entropy coefficient & 0.01 \\
Message entropy coefficient & 0.002 \\
Value function coefficient & 0.5 \\
Max gradient norm & 0.5 \\
Learning rate annealing & True \\
Target KL divergence & None \\
Batch size & 4096 \\
Minibatch size & 1024 \\
Optimizer & Adam ($\epsilon = 10^{-5}$) \\
\bottomrule
\end{tabular}
\end{table}

\begin{figure}[ht]
    \centering

    \begin{subfigure}[b]{0.32\linewidth}
         \centering
        \includegraphics[trim=5 10 5 5, clip, width=\linewidth]{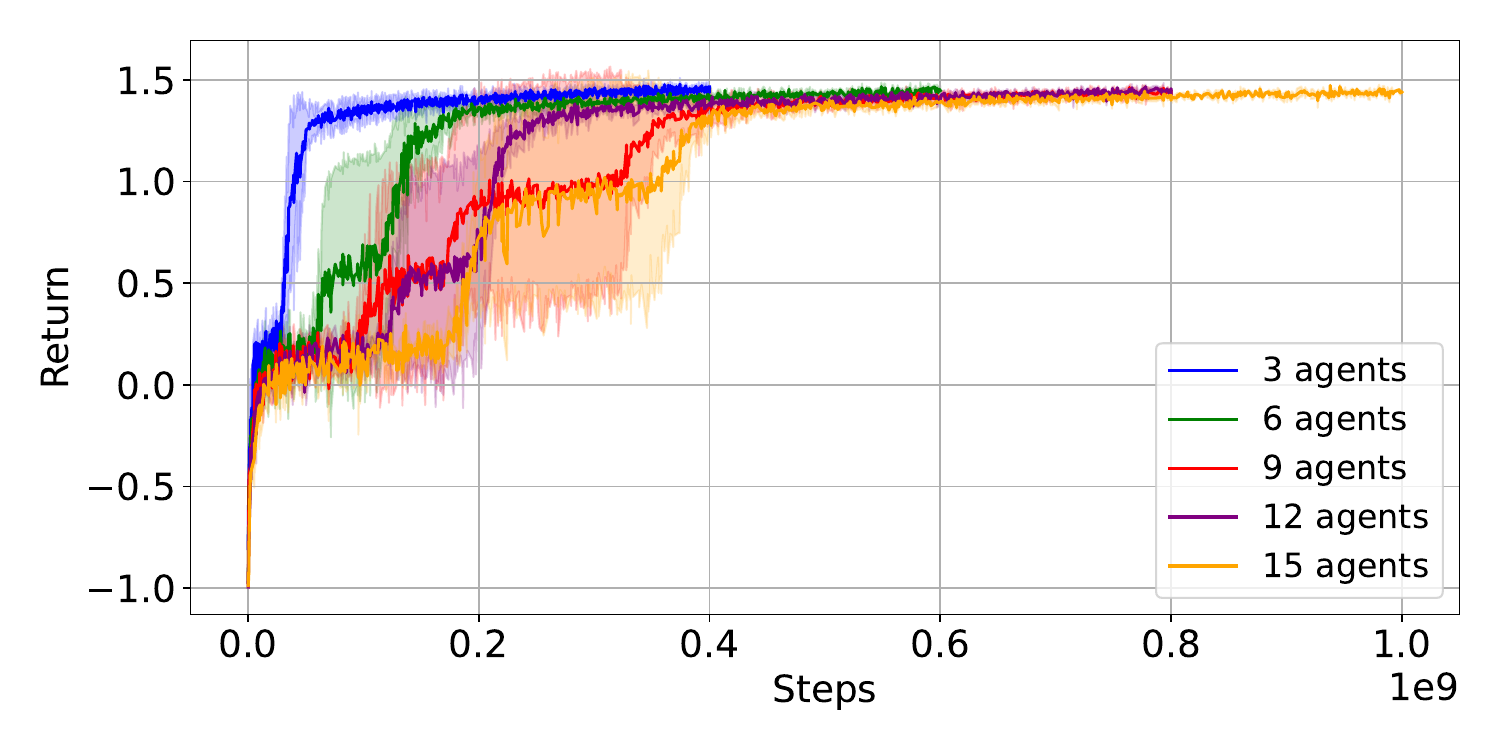}
        \caption{Episodic return}
        \label{fig:pickup_high_return}
    \end{subfigure}
    \begin{subfigure}[b]{0.32\linewidth}
        \centering
        \includegraphics[trim=5 10 5 5, clip, width=\linewidth]{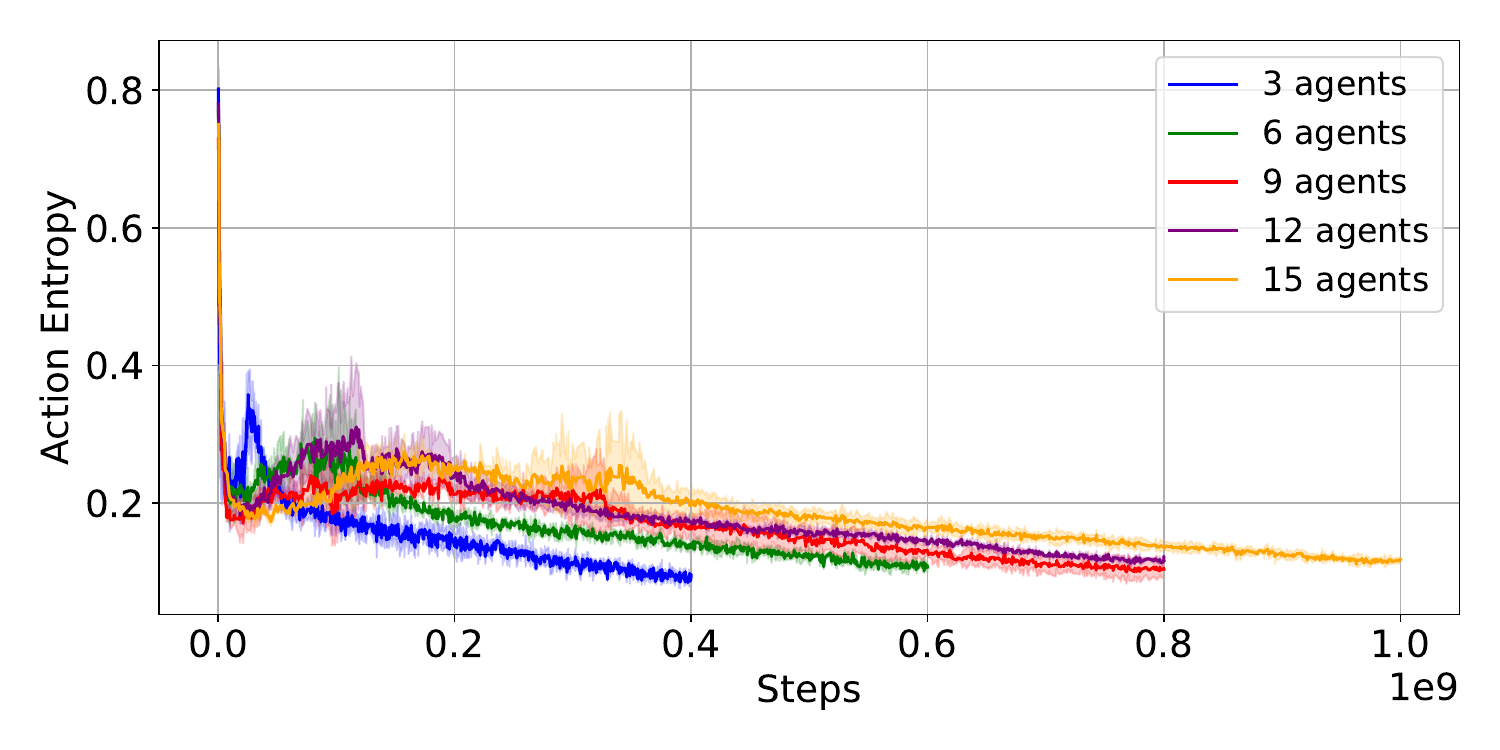}
        \caption{Action entropy}
        \label{fig:pickup_high_action_entropy}
    \end{subfigure}
    \begin{subfigure}[b]{0.32\linewidth}
        \centering
        \includegraphics[trim=5 10 5 5, clip, width=\linewidth]{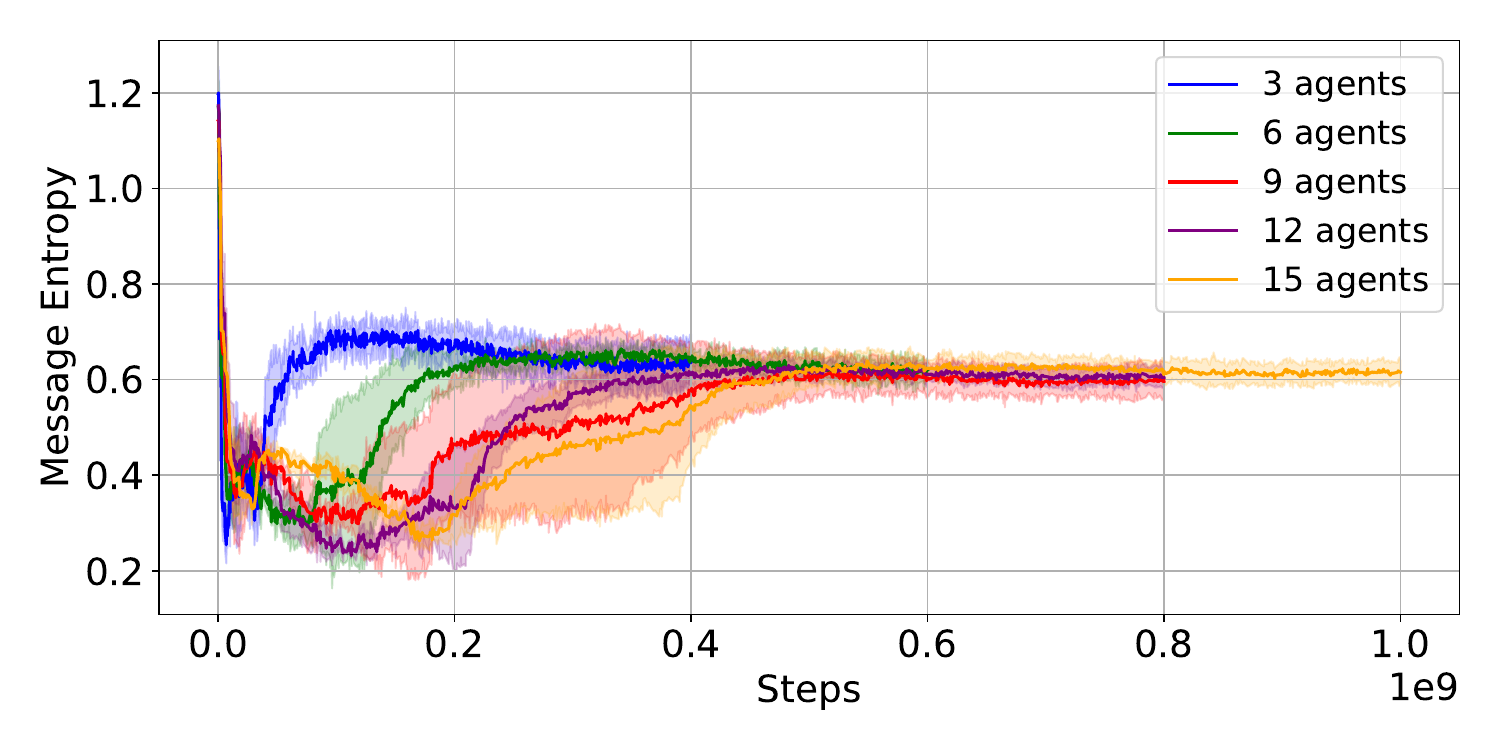}
        \caption{Message entropy}
        \label{fig:pickup_high_message_entropy}
    \end{subfigure}

    \caption{\textbf{Larger population takes longer time to converge.} The figure shows learning dynamics of agents with different population sizes in \emph{ScoreG}. Return, action entropy, and message entropy are plotted over environment steps, averaged across three seeds. A single step is defined as a single interaction with the environment.}
    \label{fig:pickup_high_learning_dynamics}
\end{figure}

\begin{figure}[ht]
    \centering

    \begin{subfigure}[b]{0.32\linewidth}
         \centering
        \includegraphics[trim=5 10 5 5, clip, width=\linewidth]{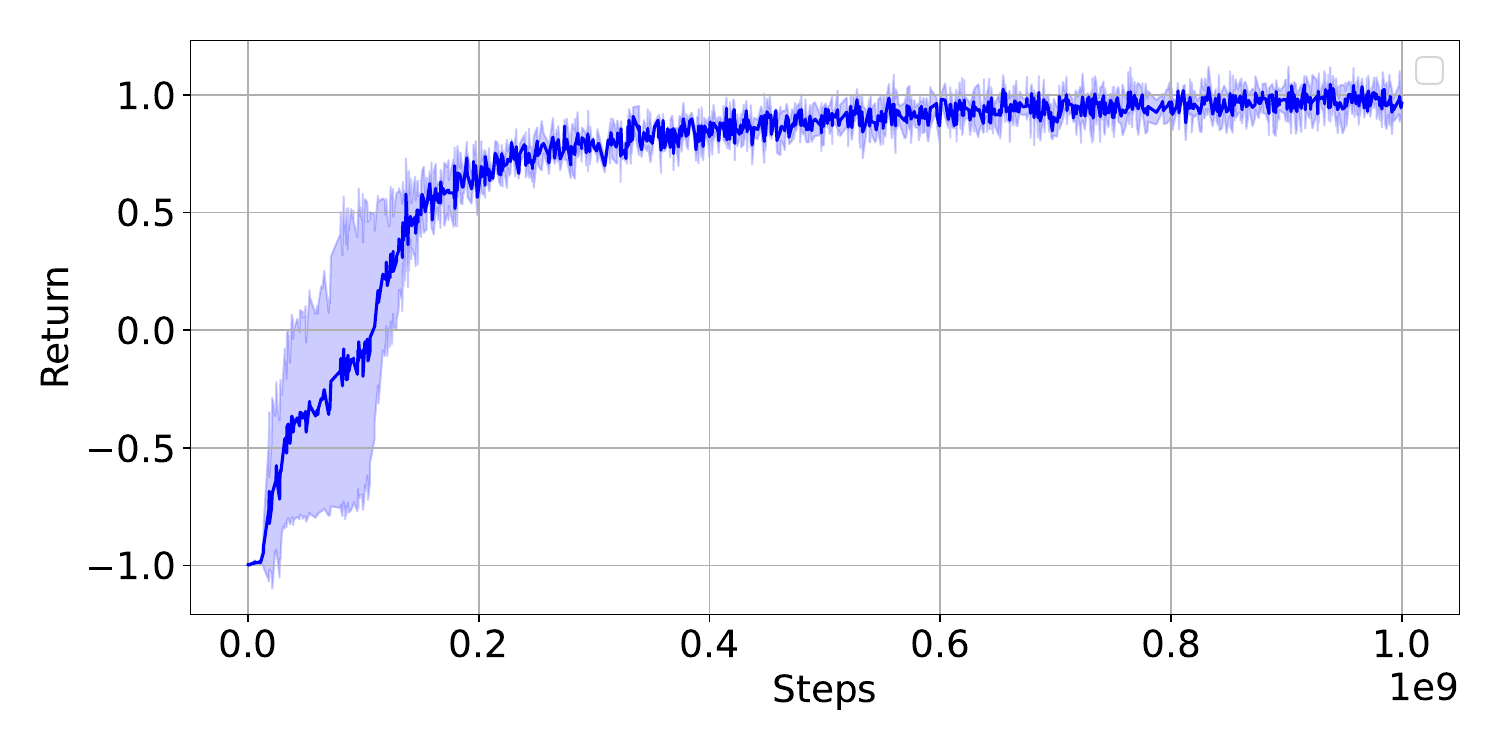}
        \caption{Episodic return}
        \label{fig:pickup_temp_return}
    \end{subfigure}
    \begin{subfigure}[b]{0.32\linewidth}
        \centering
        \includegraphics[trim=5 10 5 5, clip, width=\linewidth]{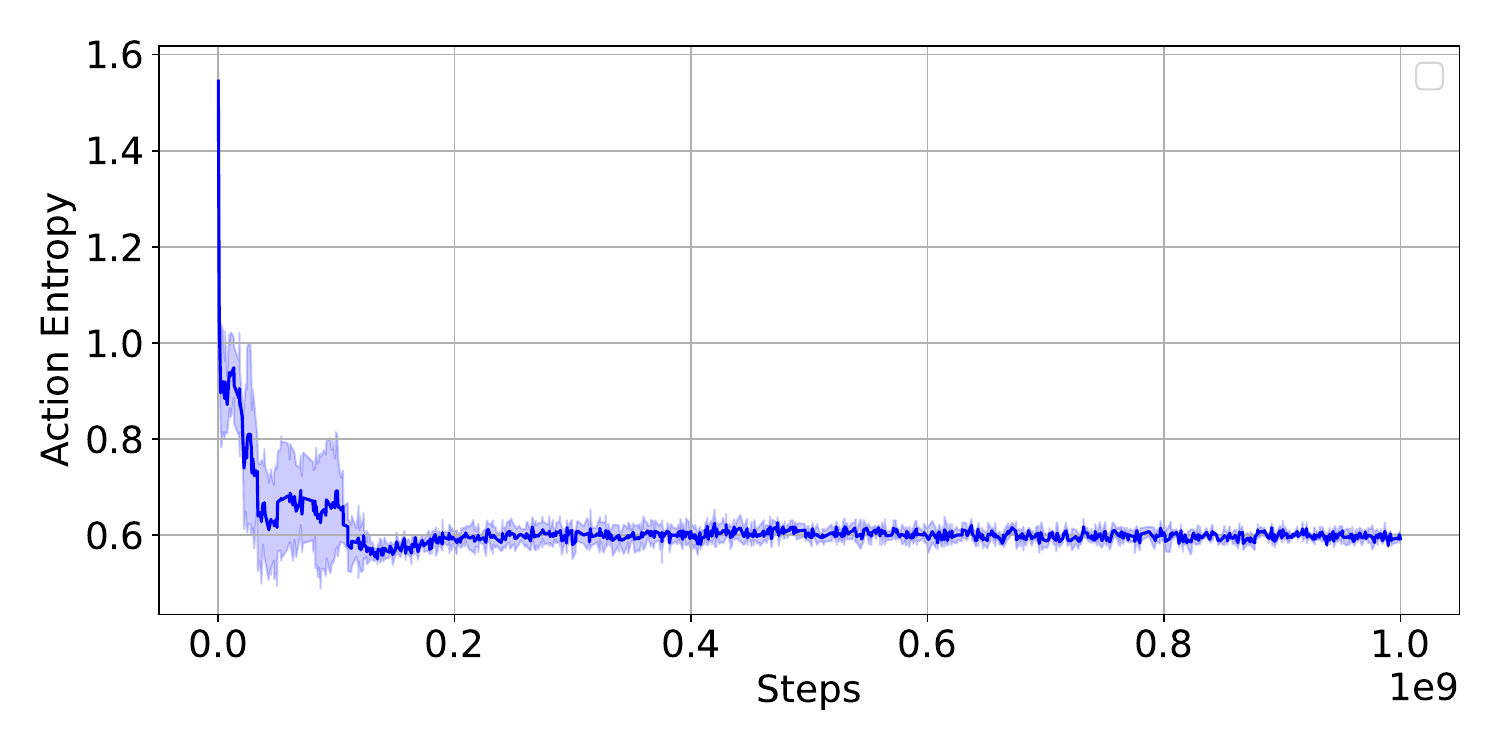}
        \caption{Action entropy}
        \label{fig:pickup_temp_action_entropy}
    \end{subfigure}
    \begin{subfigure}[b]{0.32\linewidth}
        \centering
        \includegraphics[trim=5 10 5 5, clip, width=\linewidth]{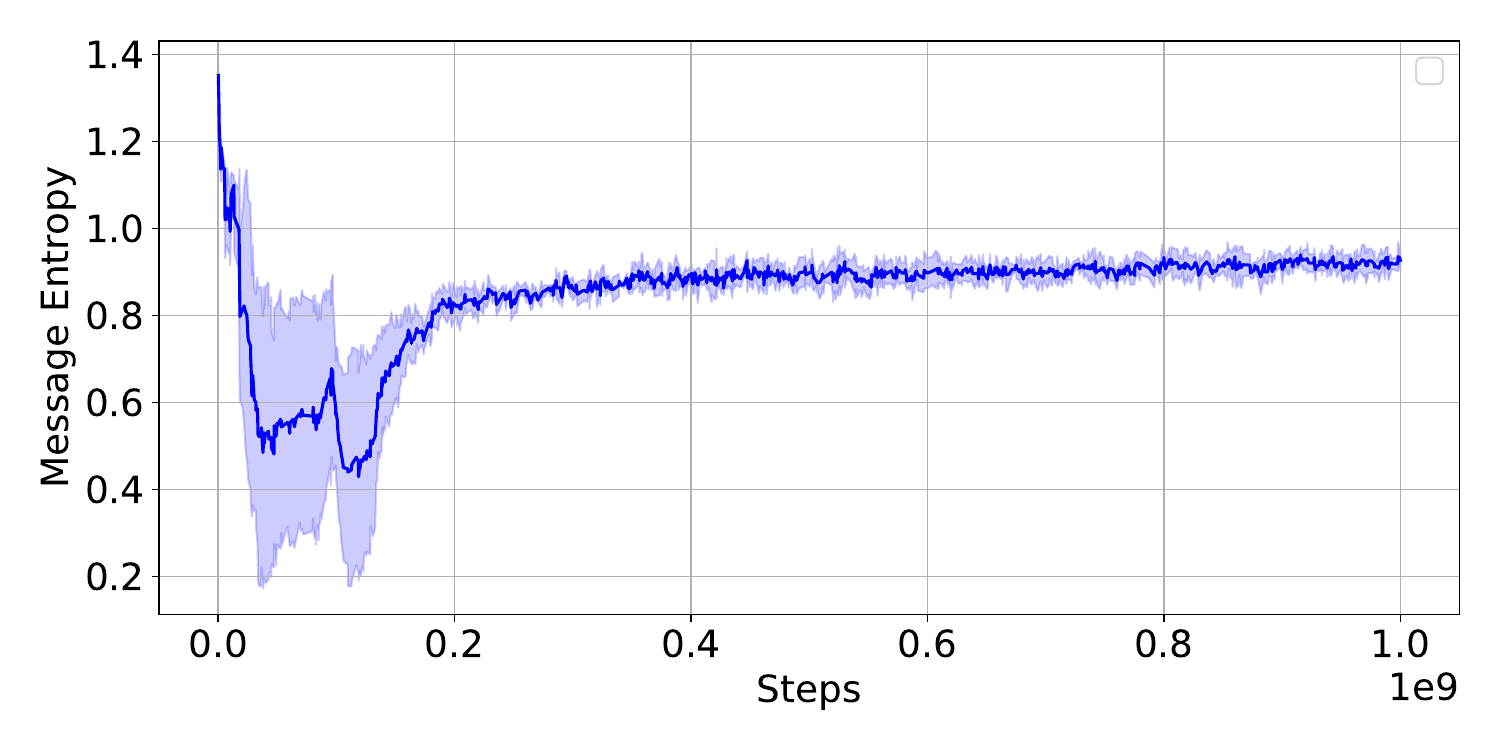}
        \caption{Message entropy}
        \label{fig:pickup_temp_message_entropy}
    \end{subfigure}

    \caption{\textbf{Learning dynamics of three XP agents in \emph{TemporalG.}} Agents start to converge at around one billion environment steps.}
    \label{fig:pickup_temp_learning_dynamics}
\end{figure}

\subsection{Metrics}
\label{appendix:metrics}
\paragraph{Topographic Similarity (topsim)} 
Topographic similarity measures the alignment between the structure of the message space and the semantic space (e.g., environmental states or properties). It is defined as the Spearman rank correlation between all pairwise distances in these two spaces \cite{brighton2006understanding, lazaridou2018emergence}. A higher \emph{topsim} indicates a more compositional and consistent mapping between semantic meanings and messages. We define the semantic space using ground-truth attributes of the environment, specifically the scores and positions of items in the \textit{ScoreG}. To implement the topographic similarity metric, we use the EGG toolkit \cite{kharitonov2019egg}.

Let \( M = \{m_1, m_2, \dots, m_N\} \) be the set of message sequences, and \( S = \{s_1, s_2, \dots, s_N\} \) be the corresponding semantic meanings. Let \( \delta_M(m_i, m_j) \) and \( \delta_S(s_i, s_j) \) denote the pairwise distance functions in the message and semantic spaces, respectively. Then:

\begin{equation}
    \text{topsim} = \mathrm{Spearman}\left(
    \left\{ \delta_M(m_i, m_j) \right\}_{i<j}, 
    \left\{ \delta_S(s_i, s_j) \right\}_{i<j}
    \right)
\end{equation}

A high topographic similarity indicates that similar messages correspond to similar semantic meanings, reflecting a structured and grounded communication protocol. We use items' positions and scores as semantic meanings.

%\textbf{Language Similarity (\emph{LS})} measures how similarly two agents communicate in the same situation. It compares the sequences of discrete messages each agent produces and computes the average agreement between them. A higher score means the agents tend to use the same messages in similar contexts, indicating stronger convergence in their communication strategies.

\paragraph{Language Similarity (LS)} quantifies the token-level similarity between two agents based on their communication over multiple episodes, starting from the same initial condition. A higher score means the agents tend to use the same messages in similar contexts, indicating stronger convergence in their communication strategies. For a given episode \( e \), the agents \( i \) and \( j \) produce their respective message sequences \( M^{(e)}_{i} \) and \( M^{(e)}_{j} \). Let \( \mathcal{D}_{\text{edit}}(\cdot) \) represent the normalized edit distance function, and let \( N_e \) denote the total number of evaluation episodes. The Language Similarity (LS) between agents \( i \) and \( j \) is defined as:

\begin{equation}
    \text{LS(i,j)} = \frac{1}{N_e}\sum_{e=1}^{N_{e}}1-\mathcal{D}_{\text{edit}}(M^{(e)}_{i},M^{(e)}_{j})
\end{equation}

We can then compute the average LS across all pairs of agents in the entire population as follows: $\text{LS} = \frac{1}{N_\text{pop}(N_\text{pop}-1)}\sum_{i=1}^{N_\text{pop}}\sum_{j \neq i}^{N_\text{pop}}\text{LS}(i,j)$.

\paragraph{Interchangeability (IC)} refers to a property of language wherein a speaker can both send and understand the same linguistic signals \cite{hockett1960origin}. In the context of agents, this means that an agent should understand the language it produces.  
To evaluate interchangeability in agents, we embed the same neural network in two different agent bodies and assess their performance in the game. We compare an agent's success when paired with itself to its success when paired with other agents.  
Formally, consider a set of $N_\text{pop}$ agents. Let $\text{SR}(i,j)$ denote the success rate of an agent $i$ when playing with an agent $j$. Therefore, we propose that interchangeability (IC) can be defined as:
\begin{equation}
    \text{IC} = (N_\text{pop} - 
 1) \times \frac{\sum_{i=1}^{N_\text{pop}} \text{SR}(i,i)}{\sum_{i=1}^{N_\text{pop}}\sum_{j \neq i}^{N_\text{pop}} \text{SR}(i,j)}
\end{equation}

\section{Theoretical Analysis}
\label{sec:theory}
We explain why three cross-play (XP) agents and two cross-play-and-self-play (XP+SP) agents tend to develop a shared language. We begin with the assumption that converged agents achieve successful coordination when their partners’ languages are compatible with them. Intuitively, if an agent is paired with a partner whose language differs more from its compatible language, their joint performance will decline.

\begin{assumption}[Coordination depends on language compatibility]
\label{ass:ls_predicts_return}
Let $\psi_1$ and $\psi_2$ be agents that achieve near-optimal performance, i.e., $(1 - \epsilon) J^* \leq J(\psi_1, \psi_2) \leq J^*$, where $J^*$ is the optimal return and $\epsilon \geq 0$ is small. Let $\psi_3$ and $\psi_4$ be alternative partners interacting with $\psi_1$. If 
\[
LS(\psi_2, \psi_3) > LS(\psi_2, \psi_4),
\]
then the joint returns satisfy:
\[
J(\psi_1, \psi_2) \geq J(\psi_1, \psi_3) \geq J(\psi_1, \psi_4).
\]
\end{assumption}

To formalize what it means for two agents to use a similar language, we introduce the following definition based on language similarity (\emph{LS}) and its relationship to joint performance.

\begin{definition}[$\epsilon$-language threshold w.r.t. a fixed partner]
\label{def:delta_l_fixed}
Let $\psi_{\text{eval}}$ and $\psi_{\text{ref}}$ be two agents such that 
\[
J(\psi_{\text{eval}}, \psi_{\text{ref}}) \geq (1 - \epsilon) J^*.
\]
Define the language compatibility threshold $\delta_l(\psi_{\text{eval}}, \psi_{\text{ref}}, \epsilon)$ as:
\[
\delta_l(\psi_{\text{eval}}, \psi_{\text{ref}}, \epsilon) := \min_{\psi'}\ LS(\psi_{\text{ref}}, \psi') \quad \text{s.t.} \quad J(\psi_{\text{eval}}, \psi') \geq (1 - \epsilon) J^*.
\]

This is the minimal language similarity (with respect to $\psi_{\text{ref}}$) that a new partner $\psi'$ must have in order to maintain near-optimal return when paired with $\psi_{\text{eval}}$.
\end{definition}

Under this setup, we can now show that language similarity within a population of XP agents ($N_\text{pop}\geq3$) is lower-bounded. This result ensures that if all agent pairs achieve near-optimal returns, their learned languages cannot be arbitrarily different.

\begin{theorem}[Language convergence in cross-play (XP) population]
\label{thm:pop_similarity}
Let $\psi_i, \psi_j, \psi_k$ be any three distinct agents from a co-trained population with $N_\text{pop}\geq3$. Suppose their pairwise joint returns are bounded:
\[
(1 - \epsilon) J^* \leq J(\psi_i, \psi_j) \leq J^*, \quad
(1 - \epsilon) J^* \leq J(\psi_j, \psi_k) \leq J^*, \quad
(1 - \epsilon) J^* \leq J(\psi_i, \psi_k) \leq J^*,
\]
for some small $\epsilon \geq 0$.

Then the language similarity between $\psi_i$ and $\psi_j$ is lower bounded as:
\[
LS(\psi_i, \psi_j) \geq \delta_l(\psi_k, \psi_i, \epsilon),
\]
where $\delta_l(\psi_k, \psi_i, \epsilon)$ is the minimum language similarity to $\psi_i$ required to achieve at least $(1 - \epsilon) J^*$ when paired with $\psi_k$.
\end{theorem}

In the same way, we can simply show that two agents trained under XP+SP have lower-bounded language similarity to maintain high task performance.

\begin{theorem}[Language convergence in cross-and-self-play (XP+SP) agents]
\label{thm:xpsp_similarity}
Let $\psi_1$ and $\psi_2$ be co-trained agents under a cross-play and self-play (XP+SP) regime, and let $J^*$ denote the optimal achievable return in the environment. Suppose both agents achieve near-optimal return when playing with each other and with themselves:
\[
(1 - \epsilon) J^* \leq J(\psi_1, \psi_2) \leq J^*, \quad
(1 - \epsilon) J^* \leq J(\psi_1, \psi_1) \leq J^*, \quad
(1 - \epsilon) J^* \leq J(\psi_2, \psi_2) \leq J^*,
\]
for some small $\epsilon \geq 0$.

Then their communication protocols must be similar:
\[
LS(\psi_1, \psi_2) \geq \max\left( \delta_l(\psi_1, \psi_1, \epsilon),\ \delta_l(\psi_2, \psi_2, \epsilon) \right),
\]
where $\delta_l(\psi_i, \psi_i, \epsilon)$ is the minimum language similarity to $\psi_i$ required to maintain return at least $(1-\epsilon) J^*$ when paired with $\psi_i$.
\end{theorem}

\begin{proof}[\textbf{Proof of Theorem~\ref{thm:pop_similarity}}]
Assume for contradiction that
\[
LS(\psi_i, \psi_j) < \delta_l(\psi_k, \psi_i, \epsilon).
\]

By definition of $\delta_l$ (Definition~\ref{def:delta_l_fixed}), this means that $\psi_j$ is less similar to $\psi_i$ than any agent known to maintain near-optimal return with $\psi_k$ when $\psi_i$ is the reference.

We consider that $\psi_k$ is the fixed agent (evaluator), $\psi_i$ is a known successful partner of $\psi_k$, and $\psi_j$ is a new partner.

Since $LS(\psi_i, \psi_j) < \delta_l(\psi_k, \psi_i, \epsilon)$, then:
\[
J(\psi_k, \psi_j) < (1 - \epsilon) J^*.
\]

This contradicts the near-optimal return condition that $J(\psi_k, \psi_j) \geq (1 - \epsilon) J^*$. Therefore,

\[
LS(\psi_i, \psi_j) \geq \delta_l(\psi_k, \psi_i, \epsilon).
\]
\end{proof}

\begin{proof}[\textbf{Proof of Theorem~\ref{thm:xpsp_similarity}}]
Assume for contradiction that:
\[
LS(\psi_1, \psi_2) < \max\left( \delta_l(\psi_1, \psi_1, \epsilon),\ \delta_l(\psi_2, \psi_2, \epsilon) \right).
\]

Without loss of generality, assume the maximum is achieved by $\psi_1$, so:
\[
LS(\psi_1, \psi_2) < \delta_l(\psi_1, \psi_1, \epsilon).
\]

From Definition~\ref{def:delta_l_fixed}, this implies that $\psi_2$ is less similar to $\psi_1$ than any agent that can achieve return $\geq (1 - \epsilon) J^*$ when paired with $\psi_1$.

We use $\psi_1$ as the fixed evaluator, $\psi_1$ and itself are known to perform well: $J(\psi_1, \psi_1) \geq (1 - \epsilon) J^*$, and $\psi_2$ is a new partner being evaluated based on its similarity to $\psi_1$.

Since $LS(\psi_1, \psi_2) < \delta_l(\psi_1, \psi_1, \epsilon)$, then
\[
J(\psi_1, \psi_2) < (1 - \epsilon) J^*,
\]
which contradicts the near-optimal return condition that $J(\psi_1, \psi_2) \geq (1 - \epsilon) J^*$.

Therefore,
\[
LS(\psi_1, \psi_2) \geq \max\left( \delta_l(\psi_1, \psi_1, \epsilon),\ \delta_l(\psi_2, \psi_2, \epsilon) \right).
\]
\end{proof}

\section{Additional Experiments and Results on
%Social Networks with Small-World Property Enhance 
Cultural Transmission}
\label{appendix:cultural_tranmission}
\begin{figure}[ht]
    \centering
    \begin{subfigure}{0.14\textwidth}
        \centering
        \includegraphics[trim=10 20 10 10, clip, width=\textwidth]{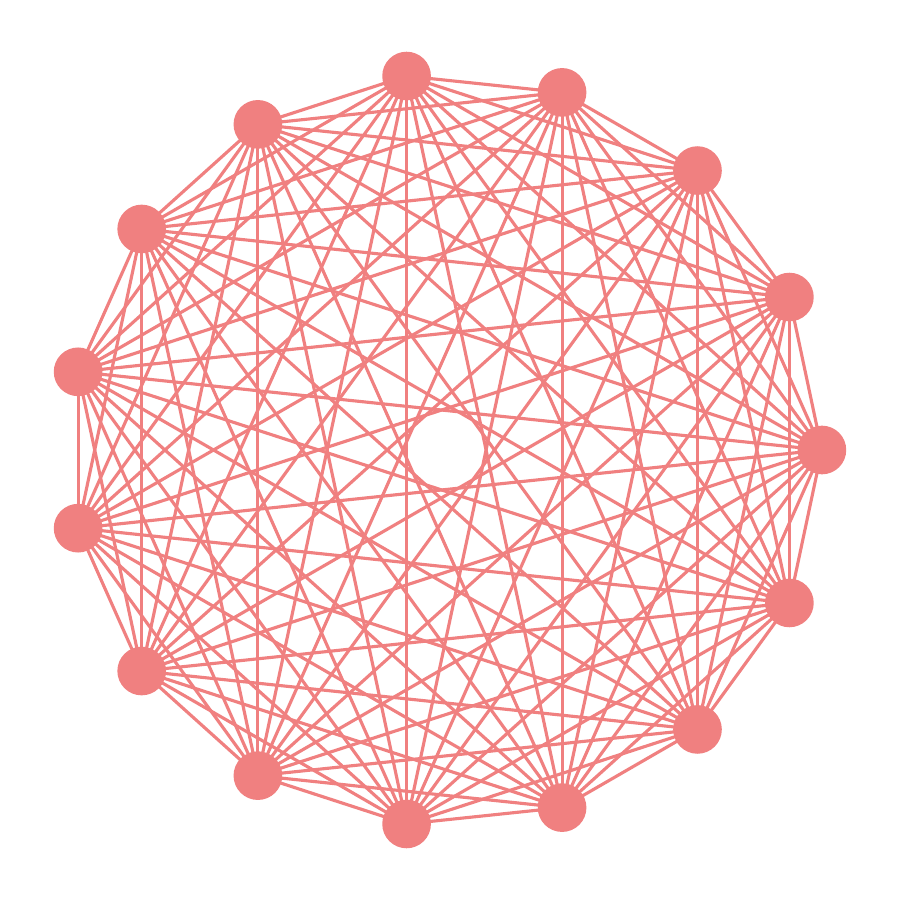}
        \vspace{-4mm}
        \caption{FC}
        \label{fig:social_fc}
    \end{subfigure}
    \begin{subfigure}{0.14\textwidth}
        \centering
        \includegraphics[trim=10 20 10 10, clip, width=\textwidth]{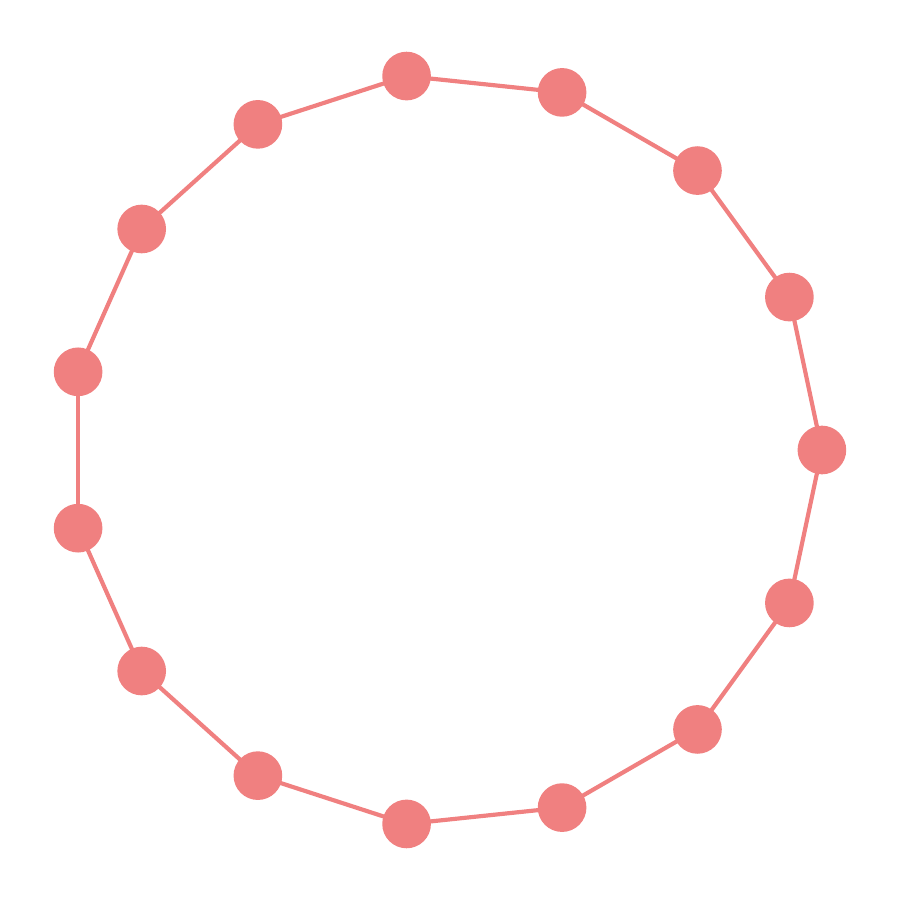}
        \vspace{-4mm}
        \caption{Ring}
        \label{fig:social_ring}
    \end{subfigure}
    \begin{subfigure}{0.14\textwidth}
        \centering
        \includegraphics[trim=10 20 10 10, clip, width=\textwidth]{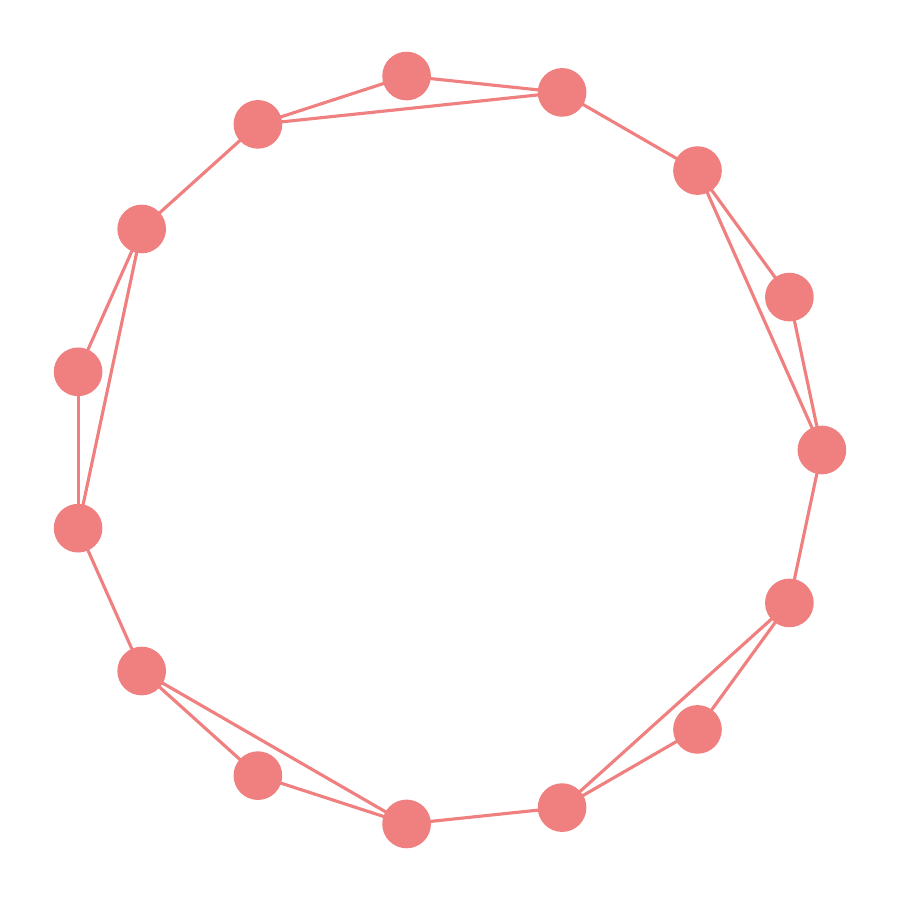}
        \vspace{-4mm}
        \caption{Clq}
        \label{fig:social_cc}
    \end{subfigure}
    \begin{subfigure}{0.14\textwidth}
        \centering
        \includegraphics[trim=10 20 10 10, clip, width=\textwidth]{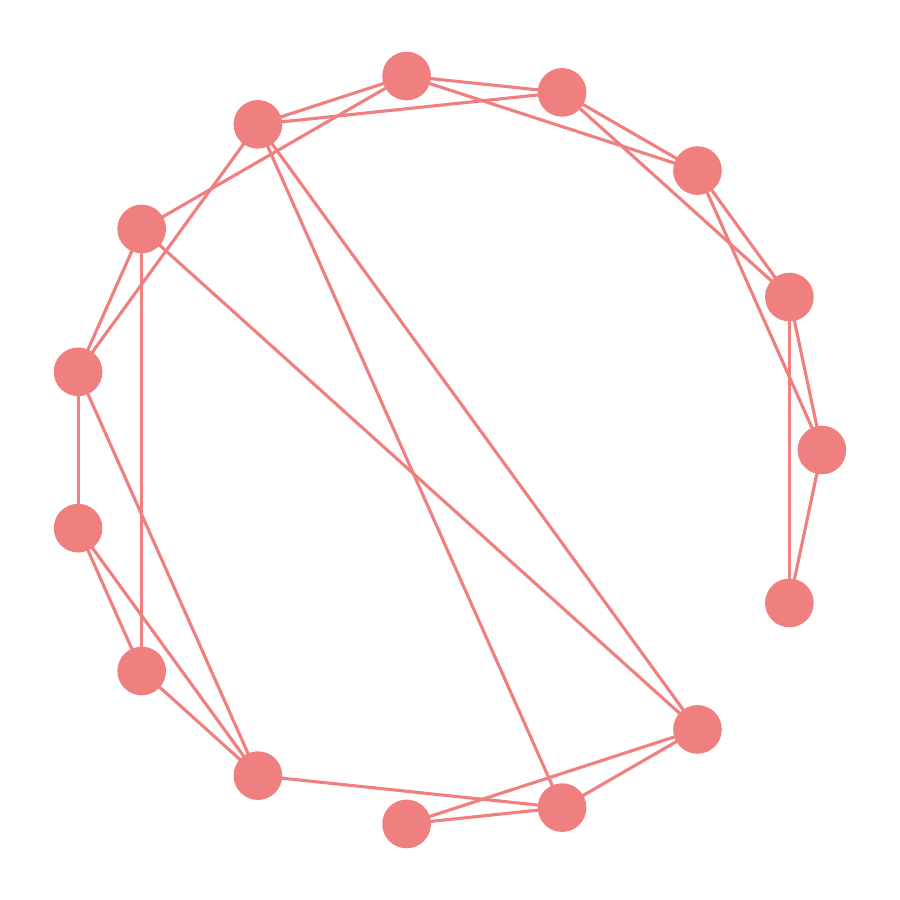}
        \vspace{-4mm}
        \caption{WS}
        \label{fig:social_ws}
    \end{subfigure}
    \begin{subfigure}{0.14\textwidth}
        \centering
        \includegraphics[trim=10 20 10 10, clip, width=\textwidth]{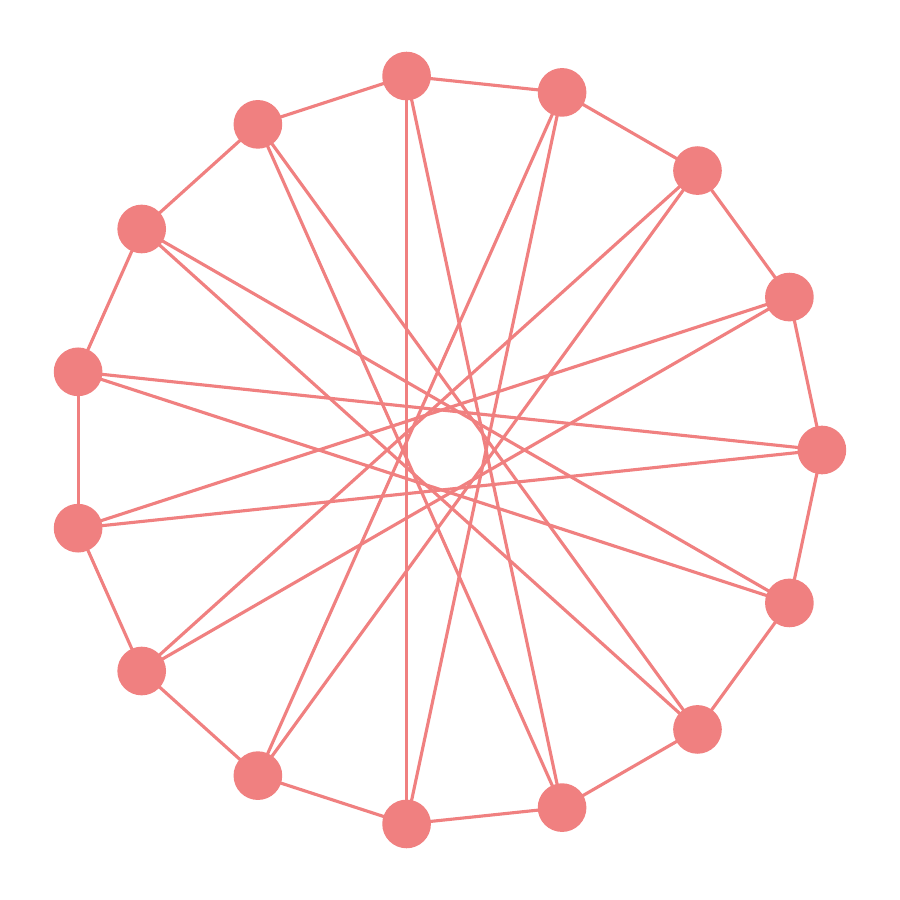}
        \vspace{-4mm}
        \caption{LRC}
        \label{fig:social_opt}
    \end{subfigure}
    \vspace{-2mm}
    \caption{\textbf{Social network structures:} Nodes represent agents, and edges denote pairs of agents that are co-trained. Network structures are abbreviated as follows: \textbf{FC – Fully Connected}, \textbf{Ring – Ring Structure}, \textbf{Clq – Ring with Cliques}, \textbf{WS – Watts-Strogatz}, and \textbf{LRC – Ring with Long-Range Connections}
    % \MM{this figure occupies too much space; and yet the content is not meaningful; we either move it to supplement; or make a nicer figure and merge it with Fig1; with each note indicating the villagers}
    }

    \label{fig:social_net}
    \vspace{-4mm}
\end{figure}

\begin{figure}[ht]
    \centering
    \begin{subfigure}[b]{0.32\textwidth}
        \centering
        \includegraphics[width=\textwidth]{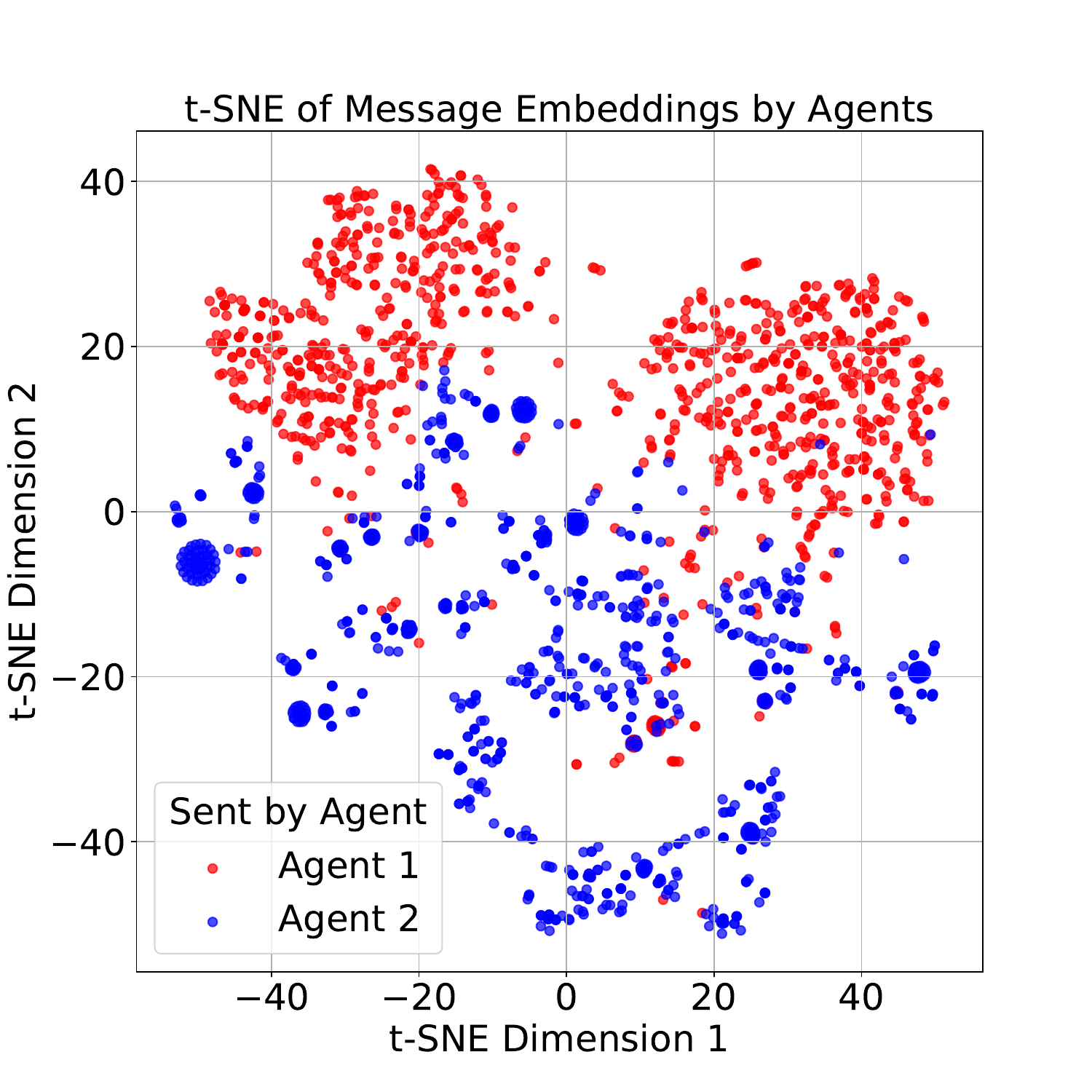}
        \caption{2 XP Agents}
        \label{fig:embedding_2xp}
    \end{subfigure}
    \hfill
    \begin{subfigure}[b]{0.32\textwidth}
        \centering
        \includegraphics[width=\textwidth]{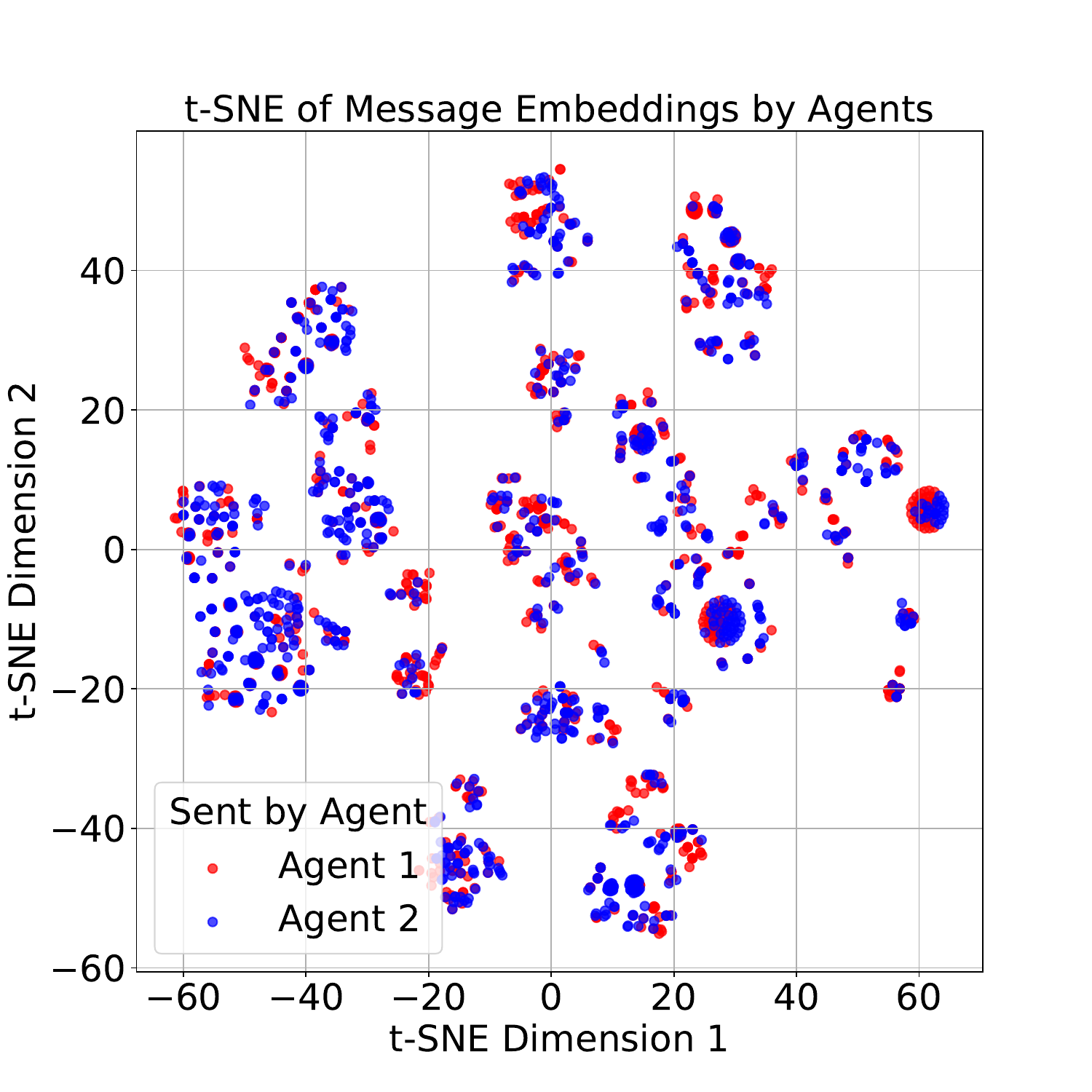}
        \caption{2 XP+SP Agents}
        \label{fig:embedding_2xpsp}
    \end{subfigure}
    \hfill
    \begin{subfigure}[b]{0.32\textwidth}
        \centering
        \includegraphics[width=\textwidth]{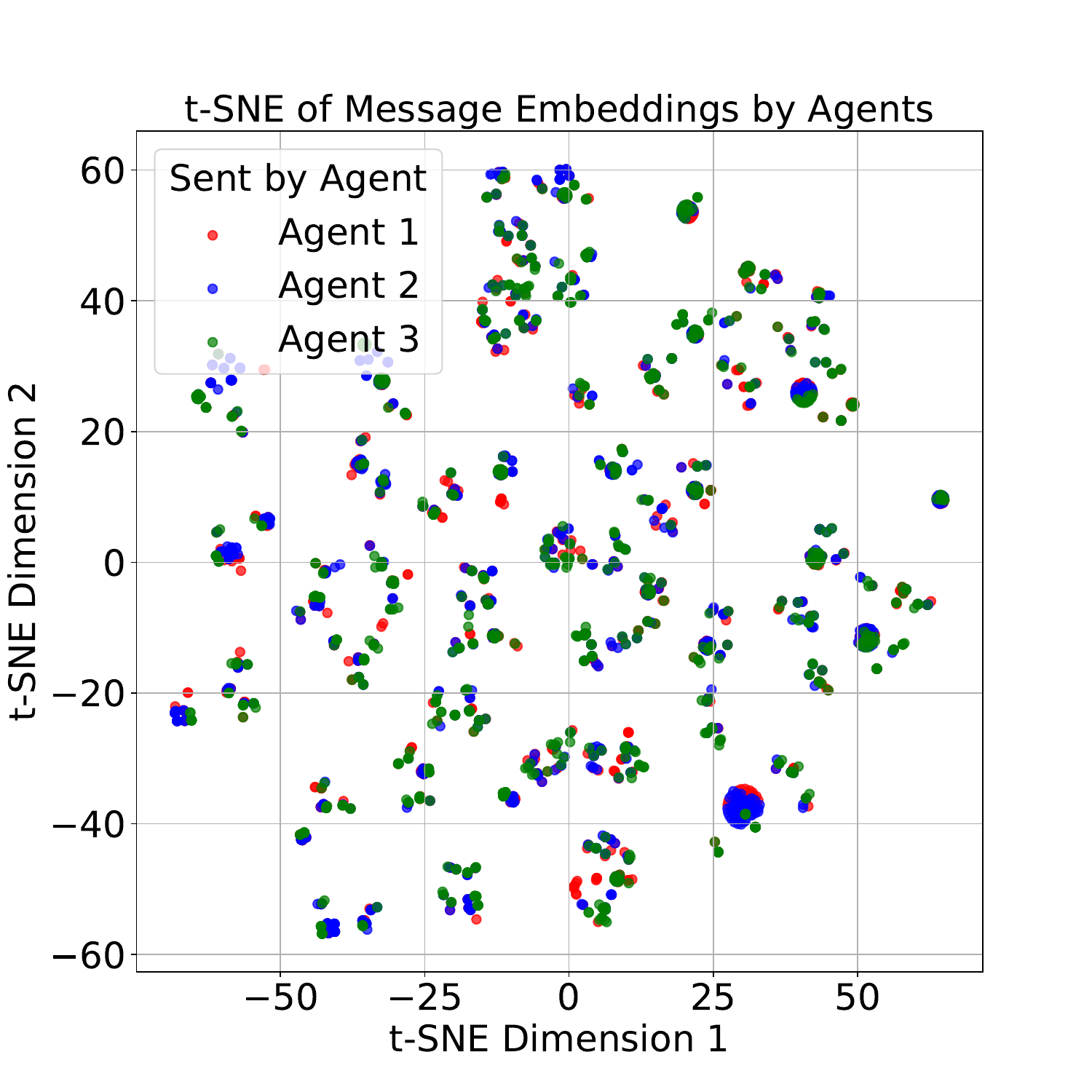}
        \caption{3 XP Agents}
        \label{fig:embedding_3xp}
    \end{subfigure}
    \caption{\textbf{Message embeddings under different training regimes.} In the t-SNE space of embedding vectors, sequences of messages from two agents are clearly separable in the 2-XP-Agents setting. (a) Messages from different agents occupy distinct regions of the embedding space (indicating different languages). (b-c) Messages from different agents largely overlap (indicating similar languages).}

    \label{fig:embedding_by_agents}
\end{figure}

\begin{figure}[ht]
    \centering
    \begin{subfigure}{0.30\textwidth}
        \centering
        \includegraphics[width=\textwidth]{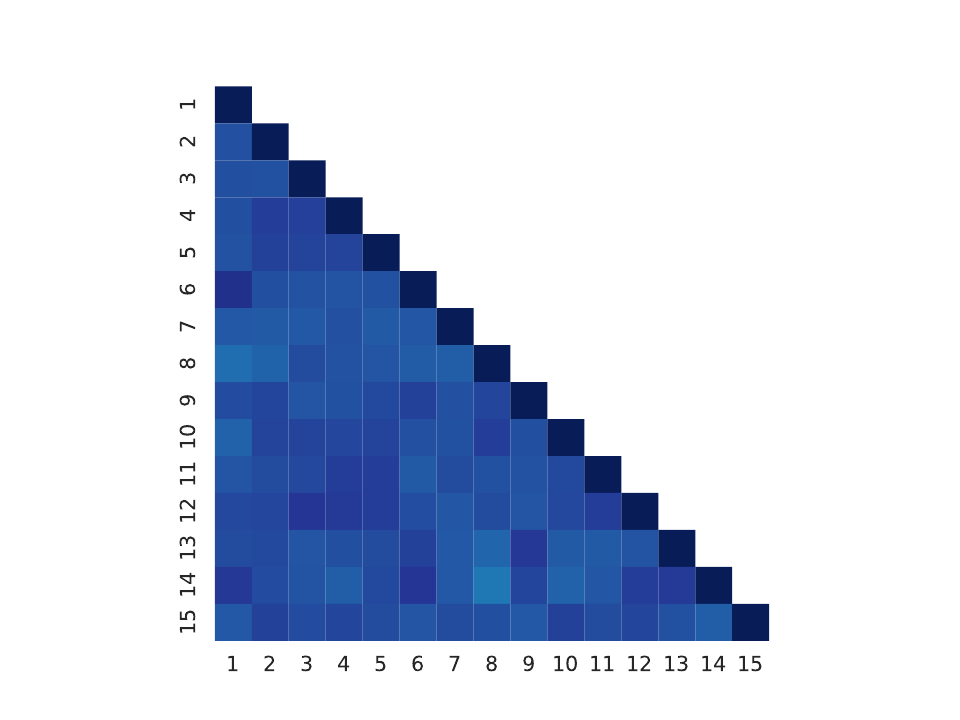}
        \caption{FC: XP}
        \label{fig:fc_xp_ls}
    \end{subfigure}
    \begin{subfigure}{0.30\textwidth}
        \centering
        \includegraphics[width=\textwidth]{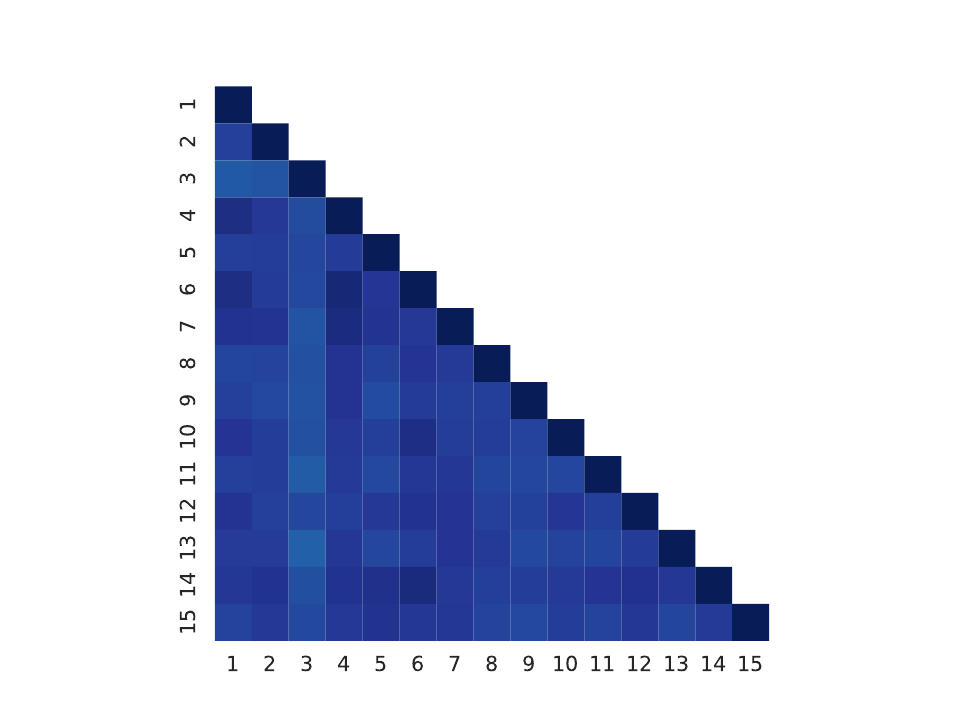}
        \caption{FC: XP+SP}
        \label{fig:fc_xpsp_ls}
    \end{subfigure}

    \vspace{0.5em}

    \begin{subfigure}{0.30\textwidth}
        \centering
        \includegraphics[width=\textwidth]{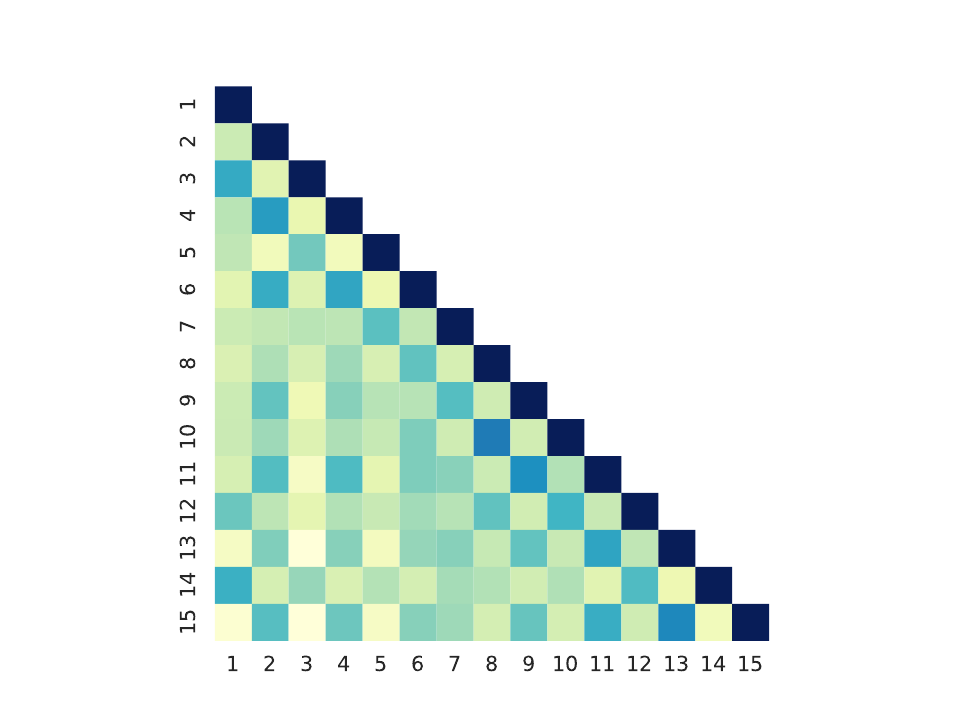}
        \caption{Ring: XP}
        \label{fig:ring_xp_ls}
    \end{subfigure}
    \begin{subfigure}{0.30\textwidth}
        \centering
        \includegraphics[width=\textwidth]{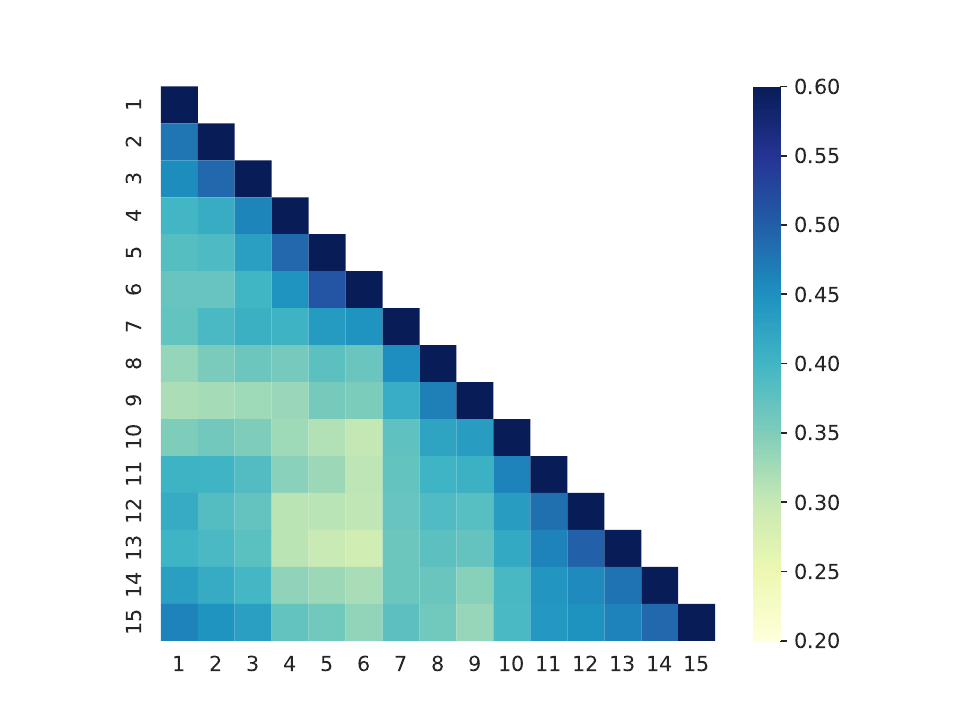}
        \caption{Ring: XP+SP}
        \label{fig:ring_xpsp_ls}
    \end{subfigure}
    
    \caption{\textbf{Language Similarity (LS) between agent pairs under different social structures and training strategies.} (a–b) Under the fully connected (FC) social network, both training regimes (XP and XP+SP) yield similar languages across agents. (c) With XP training in the Ring network, agents develop different languages from their partners (evident from the color switching between light and dark along the axis perpendicular to the diagonal). (d) With XP+SP training in the Ring network, agents and their partners converge to similar languages.
}
    \label{fig:ls_matrix}
\end{figure}

\begin{figure}[ht]
    \centering
    \begin{subfigure}{0.30\textwidth}
        \centering
        \includegraphics[trim=5 10 5 5, clip, width=\textwidth]{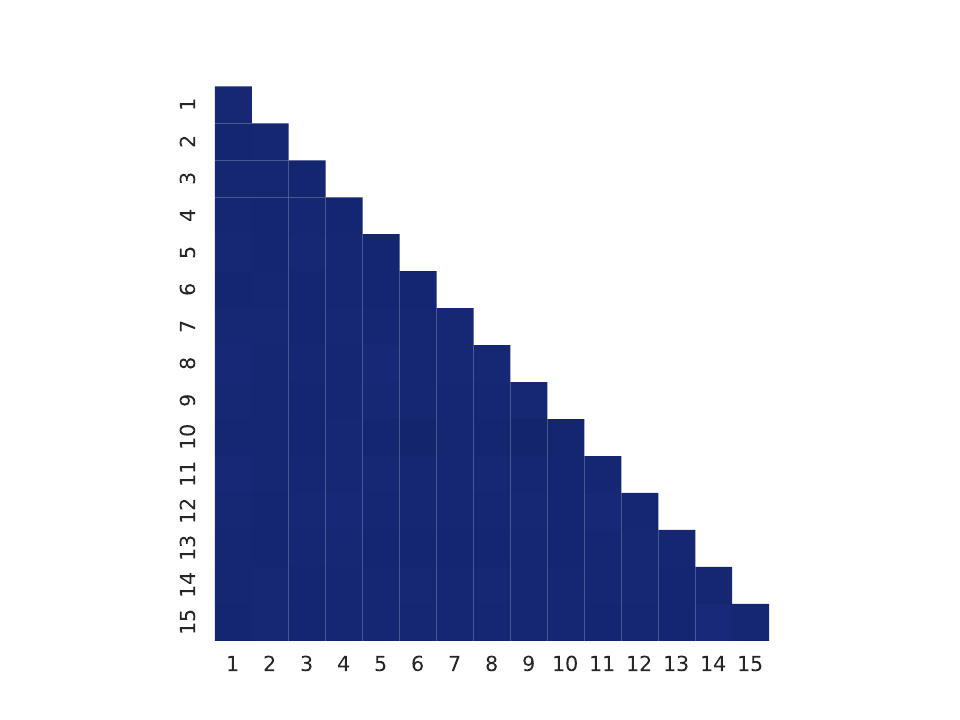}
        \caption{FC: XP}
        \label{fig:fc_xp_sr}
    \end{subfigure}
    \begin{subfigure}{0.30\textwidth}
        \centering
        \includegraphics[trim=5 10 5 5, clip, width=\textwidth]{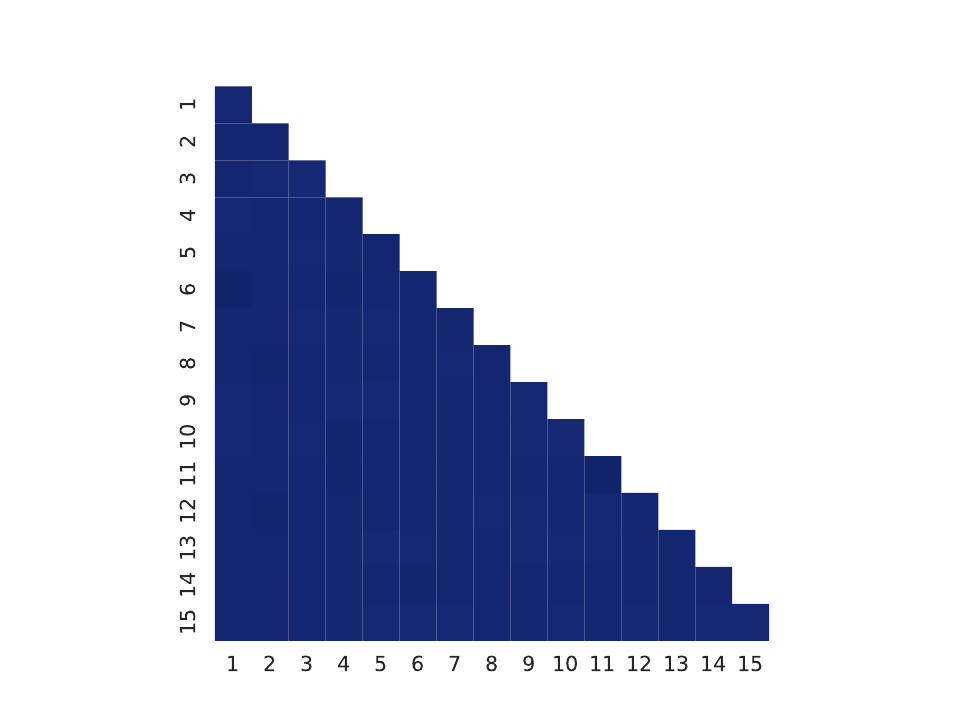}
        \caption{FC: XP+SP}
        \label{fig:fc_xpsp_sr}
    \end{subfigure}

    \vspace{0.5em}

    \begin{subfigure}{0.30\textwidth}
        \centering
        \includegraphics[trim=5 10 5 5, clip, width=\textwidth]{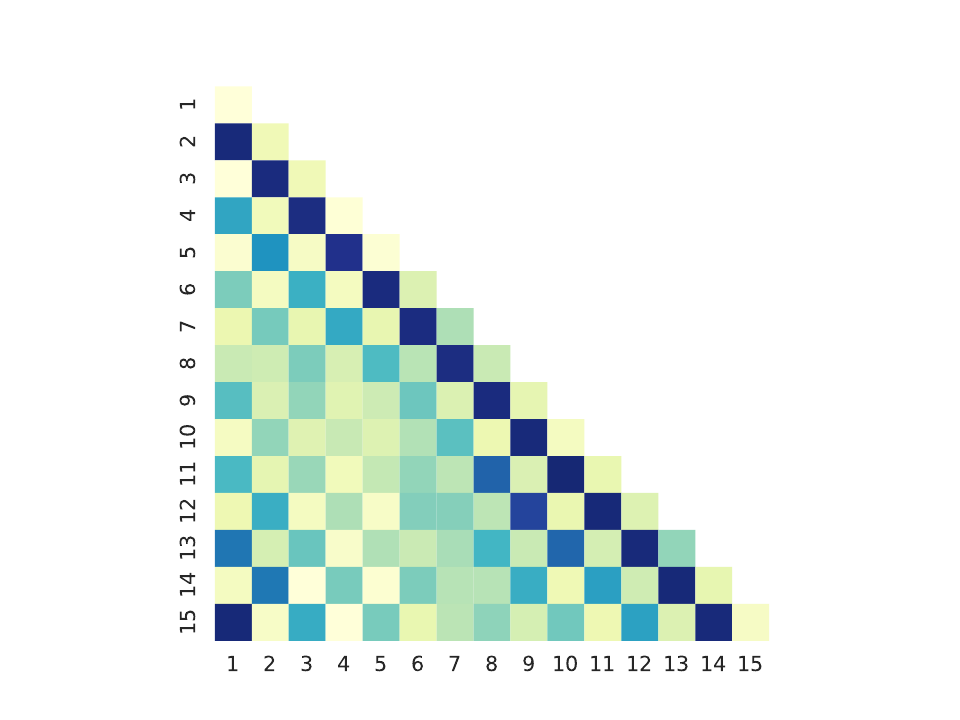}
        \caption{Ring: XP}
        \label{fig:ring_xp_sr}
    \end{subfigure}
    \begin{subfigure}{0.30\textwidth}
        \centering
        \includegraphics[trim=5 10 5 5, clip, width=\textwidth]{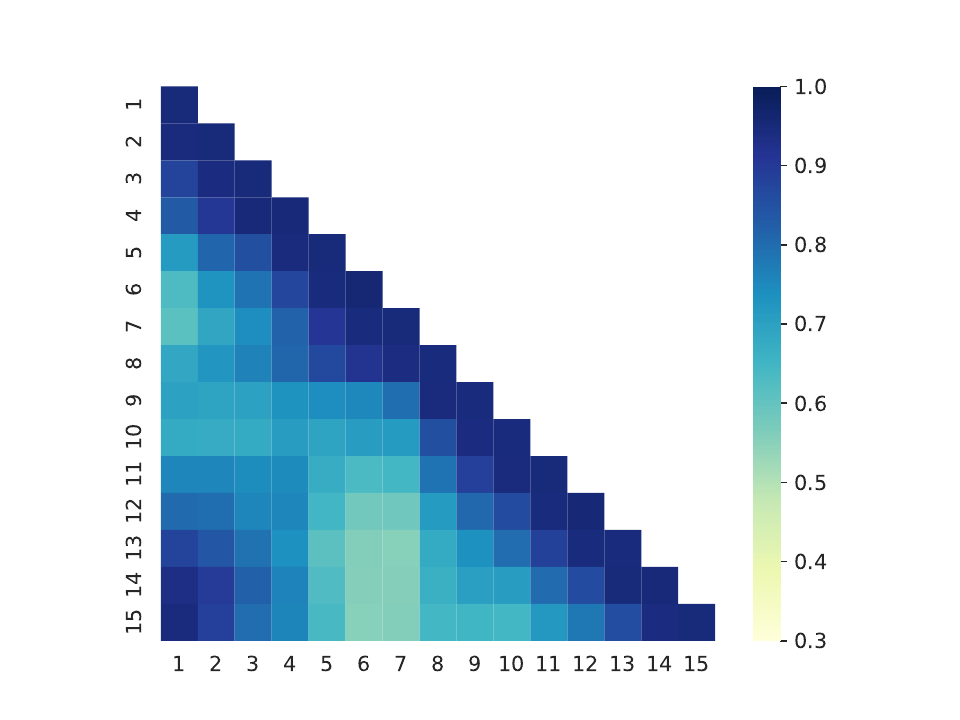}
        \caption{Ring: XP+SP}
        \label{fig:ring_xpsp_sr}
    \end{subfigure}
    
    \caption{\textbf{Success Rates (SR) between agent pairs under different social structures and training strategies.} (a–b) Both XP and XP+SP produce homogeneous SR across all pairs. (c) With XP training in the Ring network, agents fail when paired with copies of themselves. (d) With XP+SP training in the Ring network, agents succeed when paired with themselves.
}
    \label{fig:sr_matrix}
\end{figure}

\begin{figure}[ht]
    \centering
    \begin{subfigure}[b]{0.45\textwidth}
        \centering
        \includegraphics[trim=10 20 10 10, clip, width=\textwidth]{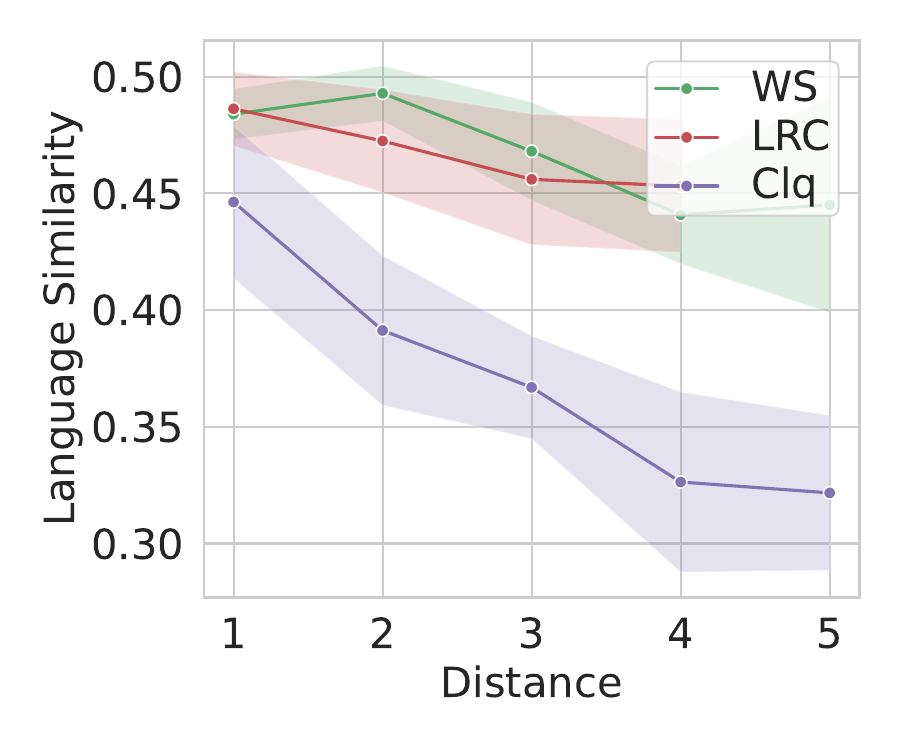}
        \caption{\emph{LS}}
        \label{fig:sm_ls_vs_dis}
    \end{subfigure}
    \begin{subfigure}[b]{0.45\textwidth}
        \centering
        \includegraphics[trim=10 20 10 10, clip, width=\textwidth]{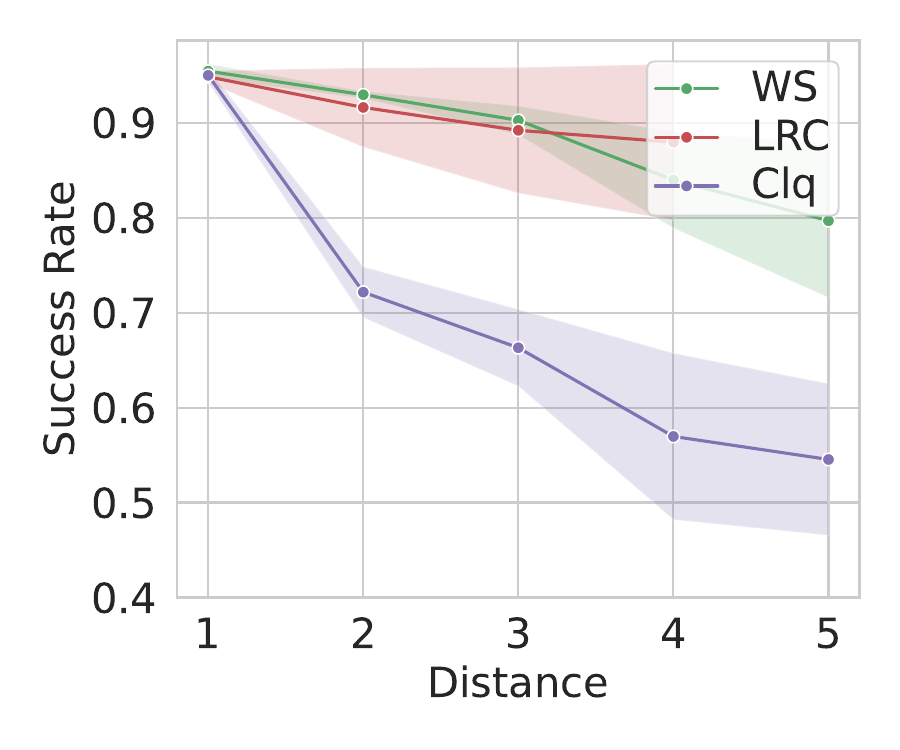}
        \caption{\emph{SR}}
        \label{fig:sm_sr_vs_dis}
    \end{subfigure}
    \vspace{-2mm}
    \caption{\textbf{More cliques and long-range connections both enhance cultural transmission ($N_{\text{pop}}=15$):} We plot \emph{LS} and \emph{SR} (See \autoref{appendix:metrics}) as a function of the shortest-path distance between two agents. A distance of 1 indicates that two agents were co-trained, while a distance greater than 1 indicates that the agents were never paired during training. Note that the maximum shortest path of LRC is only 4 while other cases are 5.
    %\MM{what happens to LRC? why stopping at distance of 4? I would plot everything up to 4 instead of 5.... trim x-axis}
    %\MT{Distance is calculated based on the shortest path. The maximum shortest path of LRC is only 4 while others' are 5.}
    }
    \label{fig:sm_ls_sr}
\end{figure}

We further examine whether adding additional connections, while keeping social networks sparse, can enhance cultural transmission. To this end, we employ two networks with small-world properties, as illustrated in \autoref{fig:social_net}. The first is the Watts-Strogatz (WS) network, constructed with parameters $k=4$ (neighbors) and $\beta=0.2$. The second is a ring structure with long-range connections (LRC) network, built upon a standard ring lattice with additional long-range edges. Specifically, the LRC network is generated by iteratively linking the farthest nodes on the ring lattice, serving as a simple proxy for reducing average path length and thereby producing a more small-world-like structure.

\textbf{More cliques and long-range connections both enhance cultural transmission [ScoreG-P15-XP]} 
We examine the impact of adding cliques or long-range connections to the ring-structured network on language transmission. Results are shown in \autoref{fig:sm_ls_vs_dis} and \autoref{fig:sm_sr_vs_dis}. For WS and LRC, both LS and SR decrease slightly with social distance, demonstrating successful language transmission to non-co-trained agents in the population. For Clq, LS at distance 1 is already lower and declines more steeply with distance than the others. SR also falls sharply, dropping from $0.95$ at distance $1$ to $0.55$ at distance $5$. This result is intuitive because Clq is constructed by adding only five connections to the ring network, which may not be enough for the population to form similar and compatible languages, whereas the others are constructed by adding many more connections to the ring network.

\section{ScoreG: Additional Results and Analyses}

\subsection{Ablation: Disembodiment (Referential Game)}
\label{appendix:disembodiment}

  %--- Left minipage: Message Analysis ---
  \begin{figure}[ht]
    \centering
    \includegraphics[trim=30 20 5 5, clip, width=0.3\textwidth]{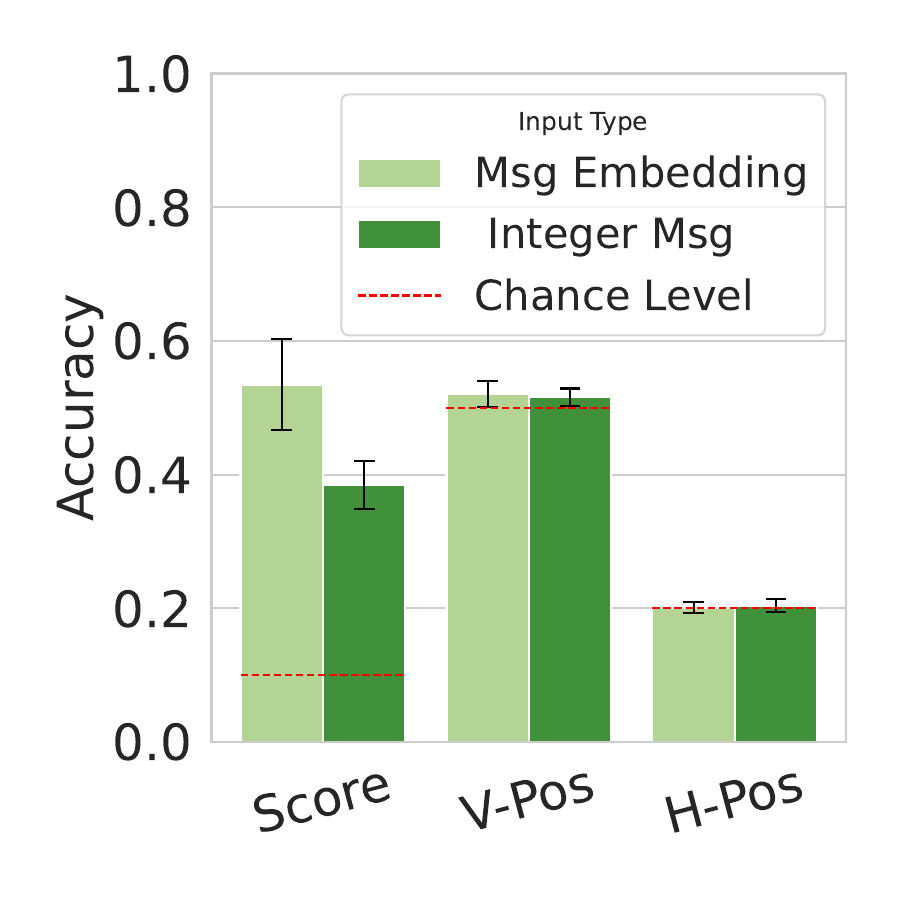}
    \vspace{-2mm}
    \caption{
      \textbf{Decoding item states from messages produced in a no-body referential game.}
      \emph{V-Pos}/\emph{H-Pos} are item positions, and \emph{Score} is item value. \emph{Integer Msg} uses raw message chains composed of integer sequences. \emph{Msg Embedding} uses chained embeddings mapped from a lookup table. Error bars show standard deviations.
    }
    \label{fig:decode_rg}
  \end{figure}

In this variant, agents are immobilized and remain fixed near their respective items. 
The game structure and goal are the same as those in \emph{ScoreG}, but without requiring agents to move to pick up the target item. Specifically, the action space is restricted to \{\textit{select my item}, \textit{select another item}, \textit{idle}\}. 
The game succeeds only if both agents correctly select the target item at the final timestep ($t=6$). 
Any premature selection causes immediate failure. 
Rewards are binary: $+1$ for success and $-1$ for failure, with no advantage for finishing early since the episode length is fixed. 

As shown in \autoref{fig:decode_rg}, agents exchange messages conveying \textit{item score} information, but not spatial information, a result confirmed by chance-level decoding accuracy. 
Each agent still observes its own static position, which in principle allows item locations to be inferred. 
However, this positional information is task-irrelevant, as agents cannot move.

\subsection{Ablation: Unidirectional Communication}
\label{appendix:unidirect}

We study a setting where only one agent (the sender) can transmit messages, while the other (the receiver) cannot send messages back—i.e., the sender receives no information from its partner. Two of the three agents are randomly selected and assigned as sender or receiver, and XP training is used. This setup reflects the unidirectionality typical of referential games.

Agents are evaluated under two conditions: \textit{Normal} (scores sampled from $[2,4,6,\dots,248]$) and \textit{Hard} (scores sampled from $[160,162,\dots,240]$). The result is shown in \autoref{tab:uni_com}. In the \textit{Normal} condition, agents achieve a high SR of $0.829$ because certain score configurations allow effective guessing. For instance, if the sender observes a low score of 32 and the receiver observes 232, the receiver infers that the sender sees 32 and can wait, while the sender correctly judges that its item is unlikely to be a target and searches for the other.

In the \textit{Hard} condition, however, SR drops to chance level ($0.508$). Because all scores are moderate, agents cannot reliably infer whether their item is a target. This result underscores the importance of bidirectional communication for achieving high SR in the ScoreG game.

\begin{table}[ht]
\centering
\caption{Performance of three XP agents with \textbf{unidirectional communication} under different test conditions.}
\label{tab:uni_com}
\begin{tabular}{lcc}
\toprule
Test Condition & SR & Length \\
\midrule
Normal & $0.829 \pm 0.022$ & $5.166 \pm 0.162$ \\
Hard   & $0.508 \pm 0.011$ & $5.934 \pm 0.503$ \\
\bottomrule
\end{tabular}
\end{table}

\subsection{Ablation: No Cooperation}
\label{appendix:no_cooperation}
\begin{table}[ht]
\centering
\caption{Ablation study on non-cooperative tasks (\emph{IndividualG}). 
We report average rewards for two agents and the average episode length 
(mean $\pm$ standard deviation) under three communication conditions.}
\label{tab:ablation_individualg}
\begin{tabular}{lccc}
\toprule
\textbf{Condition} & \textbf{Avg. Reward (First Player)} & \textbf{Avg. Reward (Second Player)} & \textbf{Avg. Length} \\
\midrule
\emph{Normal} & $1.770 \pm 0.180$ & $1.878 \pm 0.194$ & $13.762 \pm 1.240$ \\
\emph{Ablate-Noise} & $1.822 \pm 0.152$ & $1.910 \pm 0.171$ & $13.300 \pm 1.364$ \\
\emph{Ablate-Zero}   & $1.740 \pm 0.135$ & $1.945 \pm 0.179$ & $13.148 \pm 1.124$ \\
\bottomrule
\end{tabular}
\end{table}

A leading hypothesis suggests that language, or communication more broadly, emerges because agents need to share their intentions or knowledge to cooperate effectively. To test this hypothesis, we design an additional experiment in which agents are embodied in the same environment but pursue individual goals (i.e., each agent separately picks up any items in the map). In this setting, there is no mutual incentive for agents to pick up the same item—in other words, shared intentionality is removed. We use the same neural network architecture and training regime (XP) as in the \emph{ScoreG} experiments. Agents can still send and receive messages from their partners, but we hypothesize that communication will be less useful because coordination is not required to achieve shared rewards. We describe this environment, called \emph{IndividualG}, below.

\paragraph{IndividualG} The game structure is similar to \emph{ScoreG}, with agents placed in a $6 \times 6$ grid world and sharing the same action space. However, there are several key differences. First, no scores are assigned to items, and each item requires only a single agent to pick it up. Consequently, no shared reward is available to cooperative agents, and communication is unnecessary to succeed. Second, an agent receives a reward of +1 for each item it successfully collects. Items cannot be picked up by two agents simultaneously. There are four items in total, so the maximum reward for a single agent is 4. The Nash equilibrium is for each agent to pick up two items. The game ends after 20 time steps or once all items have been collected.

\paragraph{Results} The goal of this experiment is to test whether communication is useful in non-cooperative tasks. In other words, can meaningful communication still emerge when cooperation is unnecessary? We train three XP agents and evaluate them under three conditions: \emph{Normal}, \emph{Ablate-Zero}, and \emph{Ablate-Noise}. In \emph{Ablate-Zero}, agents’ messages are replaced with zero tokens, preventing the transmission of meaningful signals. In \emph{Ablate-Noise}, agents’ messages are replaced with random samples drawn from a uniform distribution instead of the policy’s output distribution. This setup forces agents to misinterpret their partners’ messages if the communication learned during training had any effect. In each episode, two of the three agents are randomly assigned as the \textbf{First Player} and \textbf{Second Player}. We report the average reward (corresponding to the average number of items collected per episode) for both players to test for potential asymmetries in their behavior. Results are shown in \autoref{tab:ablation_individualg}. Across all three conditions, agents achieve similar performance, with no observable difference between the first and second players. This suggests that communication does not provide a measurable advantage when agents pursue independent rewards. In other words, without shared goals or incentives for coordination, the environment does not pressure agents to develop or rely on meaningful communication.

\subsection{Analysis: Varying Vocabulary Size}
In \emph{ScoreG}, the agents typically complete the game in approximately 5-6 time steps because we introduce time pressure to encourage agents to finish the game as soon as possible. Varying the vocabulary size alters the capacity of information that can be conveyed through the message sequence. Here we train XP agents with $N_\text{pop}=3$ with different vocabulary sizes. The default vocabulary size we used in the main text was 4 (i.e., $\mathbf{m_t} \in \{0,1,2,3\}$). A smaller vocabulary should, in principle, encourage greater compositionality due to a smaller capacity. As shown in \autoref{tab:vocab_comparison}, compositionality remains relatively stable across vocabulary sizes of $4$, $8$, and $16$. However, when the vocabulary size increases to $32$, the \emph{topsim} score drops to $0.25$, indicating a notable decline in compositional structure.
\begin{table}[ht]
\centering
\caption{\textbf{Ablation Study on Vocabulary Size.} Effect of vocabulary size on communication and performance metrics. Values are reported as mean ± standard deviation.\\}
\label{tab:vocab_comparison}
\begin{tabular}{ccccc}
\toprule
\textbf{Vocabulary Size} & \textbf{topsim} & \textbf{IC} & \textbf{Self-SR} & \textbf{Cross-SR} \\
\midrule
4  & $0.311 \pm 0.111$ & $0.966 \pm 0.017$ & $0.939 \pm 0.021$ & $0.972 \pm 0.004$ \\
8  & $0.352 \pm 0.032$ & $0.902 \pm 0.056$ & $0.880 \pm 0.057$ & $0.975 \pm 0.003$ \\
16 & $0.309 \pm 0.052$ & $0.954 \pm 0.030$ & $0.933 \pm 0.030$ & $0.978 \pm 0.001$ \\
32 & $0.251 \pm 0.034$ & $0.954 \pm 0.031$ & $0.931 \pm 0.031$ & $0.977 \pm 0.001$ \\
\bottomrule
\end{tabular}
\end{table}

\subsection{Generalization to Unseen Positions}
\label{appendix:gen_unseen_pos}

We study how well communication generalizes when it comes to unseen object positions. We trained three XP agents to perform the ScoreG task on 17 food locations in a 5×5 grid and evaluated their generalization to 8 unseen locations. Train and test food locations are randomly selected. The results are shown in \autoref{tab:generalize_loc}. While agents with communication perform well with SR of $0.953$ on seen locations, SR drops to $0.756$ when agents are tasked to pick up objects at the unseen locations.
We further ask whether the observed generalization gap arises from communication itself or from limitations of the underlying RL method. To test this, we train three XP agents that can observe all scores and each other but cannot communicate—essentially a standard MARL setup. The agents are trained on 17 item locations and evaluated on 8 unseen locations. As shown in \autoref{tab:generalize_loc_no_comm}, SR drops from $0.981$ to $0.823$, indicating that the gap stems from the inherent limits of the MARL algorithm rather than communication alone. Developing novel methods to improve the generalization of multi-agent coordination under unseen spatial configurations is an important direction for both MARL in general and emergent communication in embodied settings. However, this lies beyond the scope of the present study and we leave it for future work.

\begin{table}[ht]
\centering
\caption{Performance of three XP agents with communication on seen vs unseen locations.}
\begin{tabular}{ccc}
\toprule
{Location} & {SR} & {Length} \\
\midrule
Seen    & $0.953 \pm 0.017$ & $5.197 \pm 0.099$ \\
Unseen  & $0.756 \pm 0.022$ & $5.154 \pm 0.229$ \\
\bottomrule
\end{tabular}
\label{tab:generalize_loc}
\end{table}

\begin{table}[ht]
\centering
\caption{Performance of three XP agents without communication on seen vs unseen locations.}
\label{tab:seen_unseen}
\begin{tabular}{ccc}
\toprule
Location & SR & Length \\
\midrule
Seen   & $0.981 \pm 0.006$ & $4.389 \pm 0.038$ \\
Unseen & $0.823 \pm 0.023$ & $4.092 \pm 0.062$ \\
\bottomrule
\end{tabular}
\label{tab:generalize_loc_no_comm}
\end{table}

\subsection{Analysis of Below-Chance Success Rates in Implicit Communication}
\label{appendix:implicit-sr}

In \emph{Inv-NoCom} setting, the success rate (SR) falls below the chance level (See \autoref{fig:ablation_sr}), 
which indicates a systematic coordination failure rather than random performance. 
To better understand this phenomenon, we conduct two analyses: (i) ablation of the 
speed reward, and (ii) categorization of failure outcomes.

\paragraph{Ablation of the speed reward.}  
The speed reward encourages efficient task completion. Removing this reward does not 
substantially change SR, which remains below chance. However, the absence of the speed 
reward increases the average episode length (Table~\ref{tab:speed-reward}). 
This result shows that the speed reward primarily affects efficiency but does not resolve 
the core coordination issue that leads to low SR.

\begin{table}[ht]
\centering
\caption{Effect of the speed reward on success rate (SR) and average episode length.}
\label{tab:speed-reward}
\begin{tabular}{lcc}
\toprule
Reward Condition & SR & Length \\
\midrule
With speed reward    & 0.438 & 5.607 \\
Without speed reward & 0.446 & 7.950 \\
\bottomrule
\end{tabular}
\end{table}

\paragraph{Failure outcome statistics.}  
To diagnose why agents fail, we categorize episode outcomes into three groups: 
\emph{correct pickups}, \emph{wrong pickups}, and \emph{no pickups}. 
Table~\ref{tab:failure-categories} shows the distribution. Failures are dominated 
by wrong pickups (34.8\%), followed by cases where agents never pick up the same item (27.4\%). 
Together, these account for the below-chance success rate and show that agents struggle 
to coordinate reliably without communication.

\begin{table}[ht]
\centering
\caption{Distribution of episode outcomes without the speed reward.}
\label{tab:failure-categories}
\begin{tabular}{lc}
\toprule
Outcome & Percentage \\
\midrule
Success: Correct Pickup & 37.8\% \\
Failed: Wrong Pickup    & 34.8\% \\
Failed: No Pickup       & 27.4\% \\
\bottomrule
\end{tabular}
\end{table}

Overall, these findings show that below-chance SR in \emph{Inv-NoCom} setting does not result from the speed reward but from miscoordination, 
most prominently through wrong item pickups.

\subsection{ScoreG: Bumping and Cell Occupancy}
\label{appendix:scoreg_bumping}

Each grid cell can be occupied by at most one entity 
(either an agent or an item) at any given time. If an agent attempts to move into a cell 
that is already occupied by another agent, the move is blocked and the agent remains in its 
original position. We refer to such blocked moves as \emph{bumping events}.

Although bumping could theoretically provide a minimal form of implicit communication 
(e.g., signaling another agent’s location), it constitutes only a very low-bandwidth channel and becomes difficult to coordinate once agents are spatially separated. In practice, bumping is rare and therefore does not play a substantive role in coordination or communication. \autoref{tab:bumps} summarizes the average number of bumps observed per episode across different conditions.

\begin{table}[ht]
\centering
\caption{Average bumps per episode.}
\label{tab:bumps}
\begin{tabular}{lcc}
\toprule
Condition & Success & Failed \\
\midrule
Inv-Com   & 0.179   & 0.065  \\
Inv-NoCom & 0.043   & 0.057  \\
Vis-NoCom & 0.015   & 0.020  \\
\bottomrule
\end{tabular}
\end{table}

These results confirm that while bumping exists as an artifact of the grid-world dynamics, it is too infrequent ($<1$ bump per episode) to support reliable strategies or emergent communication.

\section{TemporalG: Additional Results and Analyses}

\begin{table}[ht]
\centering
\small
\caption{
Performance of XP agents in the $5 \times 5$ \emph{TemporalG} environment with the fixed duration of 6. 
$N_{\text{pop}}$: population size. LS: language similarity; Cross-SR: success rate with other agents; Self-SR: success rate in self-play. 
All values are reported as mean ± standard deviation across held-out agent pairs.
}
\begin{tabular}{lcccc}
\toprule
\textbf{Model} & \textbf{$N_{\text{pop}}$} & \textbf{LS} & \textbf{Cross-SR} & \textbf{Self-SR} \\
\midrule
XP & 3 & $0.336 \pm 0.012$ & $0.975 \pm 0.003$ & $0.962 \pm 0.006$ \\
\bottomrule
\end{tabular}

\label{tab:temporal_order_performance}
\end{table}

\subsection{Implicit Communication}
\label{appendix:implicit_temporalg}
As shown in \autoref{tab:tempg_8x8_ablate}, three XP agents trained without communication under the invisible-partner condition (\emph{Inv-NoCom}) still achieved relatively high success at a fixed duration of 6, likely by adopting a waiting strategy—that is, delaying movement to infer the spawn order. Their longer episode lengths, compared to the other two conditions (\emph{Inv-Com} and \emph{Vis-NoCom}), support this interpretation. However, when the fixed duration was increased, the number of possible spawn times doubled (from 6 to 12), rendering the waiting strategy ineffective and causing the success rate of \emph{Inv-NoCom} agents to fall below chance level. Similar to what was observed in \emph{ScoreG} experiment in \autoref{sec:results}, \emph{Vis-NoCom} agents achieved performance comparable to \emph{Inv-Com}, suggesting that they exploited visual cues for implicit coordination, especially under the limited set of six possible spawn times. In the \emph{Inv-Com-Noise} condition, agents were trained with explicit communication but tested without it (messages were always zero). The low SR observed in this condition implies that meaningful messages are necessary for agents trained with communication (\emph{Inv-Com}) to succeed at the task.

\begin{table}[ht]
\centering
\caption{Implicit communication in the
$8 \times 8$ \emph{TemporalG} with fixed durations of 6 and 12. Length denotes the average length of successful episodes (not applicable in the bottom-right cell, where some runs yielded zero successes).
Abbreviations are the same as those reported in \autoref{fig:ablation_all}.}
\label{tab:tempg_8x8_ablate}
\begin{tabular}{lcccc}
\toprule
\multirow{2}{*}{Condition} & \multicolumn{2}{c}{Fixed Duration = 6} & \multicolumn{2}{c}{Fixed Duration = 12} \\
\cmidrule(lr){2-3} \cmidrule(lr){4-5}
 & SR & Length & SR & Length \\
\midrule
\emph{Inv-Com}       & $0.980 \pm 0.035$ & $23.990 \pm 0.618$ & $0.961 \pm 0.021$ & $30.023 \pm 0.775$ \\
\emph{Inv-NoCom}     & $0.728 \pm 0.155$ & $28.716 \pm 2.341$ & $0.419 \pm 0.137$ & $31.738 \pm 1.780$ \\
\emph{Vis-NoCom}     & $0.939 \pm 0.093$ & $24.516 \pm 0.307$ & $0.862 \pm 0.110$ & $30.546 \pm 0.406$ \\
\emph{Inv-Com-Noise} & $0.313 \pm 0.040$ & $28.666 \pm 0.774$ & $0.003 \pm 0.001$ & --- \\
\bottomrule
\end{tabular}
\end{table}

\newpage

\subsection{Rendezvous Behavior Analysis}
\label{appendix:tempg_rendezvous}
\begin{figure}[ht]
    \centering
    \includegraphics[width=0.75\textwidth]{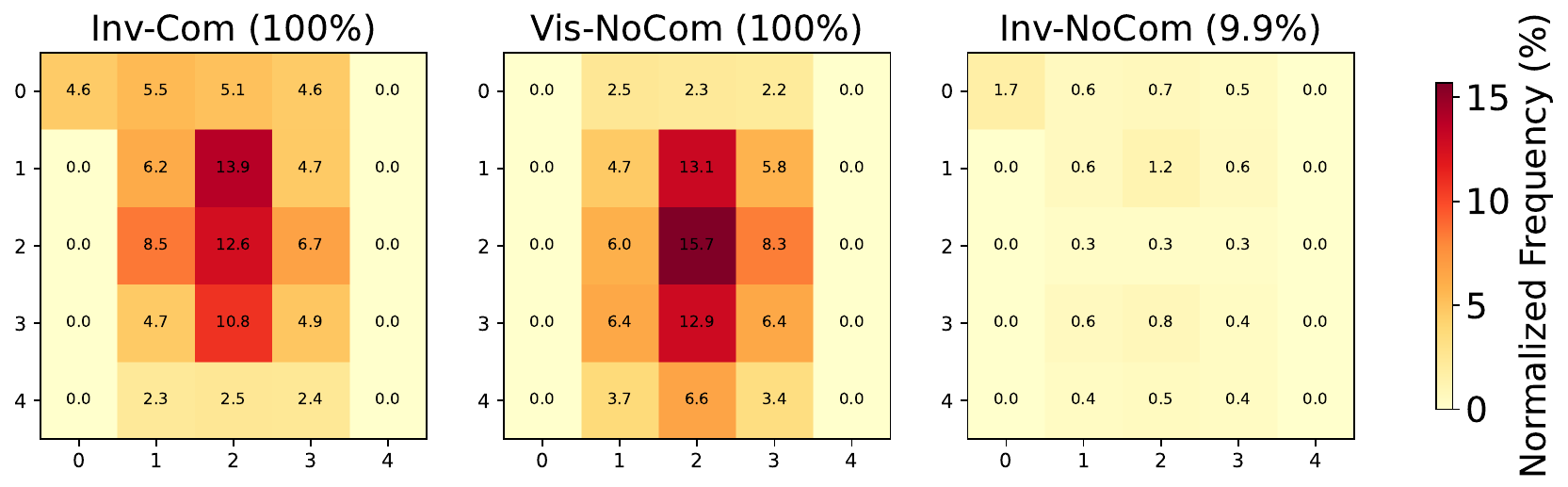}
    \caption{Rendezvous maps of first meeting points across conditions. 
    Values represent normalized frequencies of encounters on the $5 \times 5$ grid. 
    \emph{Inv-Com} and \emph{Vis-NoCom} agents consistently meet and converge near the center, 
    while \emph{Inv-NoCom} agents rarely meet and show no strong central bias.}
    \label{fig:rendezvous}
\end{figure}

Restricting the communication range forces agents to meet in physical proximity before they can exchange information. This environmental constraint gives rise to rendezvous points: locations where agents repeatedly converge to initiate communication. Despite this restriction, displacement remains necessary. Agents must explore the grid to observe items, store these observations in memory, and then communicate the information once they encounter their partner. Since neither agent can directly infer when all items have appeared without exchanging messages, coordination still relies on communication.

To analyze where meetings occur, we compute a rendezvous map that records the normalized frequency of the first meeting location across episodes. The resulting heatmaps reveal a pronounced central bias: agents overwhelmingly converge near the middle of the $5 \times 5$ grid (See \autoref{fig:rendezvous}). This bias emerges because the center minimizes travel distance and maximizes the probability of encounter under partial observability.

When comparing conditions, clear differences emerge. \emph{Inv-Com} agents always meet (100\% of episodes) and almost exclusively converge in the central cells, reflecting highly coordinated rendezvous behavior. \emph{Vis-NoCom} agents also meet in every episode, showing a similar central bias, suggesting that visibility supports coordination even without explicit messages. By contrast, \emph{Inv-NoCom} agents succeed in meeting in only 9.9\% of episodes, and when they do meet, the rendezvous points are weakly defined and scattered across the grid.

\end{document}